\documentclass[11pt]{article}

\usepackage[utf8]{inputenc} 
\usepackage[T1]{fontenc}    
\usepackage{lmodern}
\usepackage{hyperref}       
\usepackage{url}            
\usepackage{booktabs}       
\usepackage{amsfonts}       
\usepackage{nicefrac}       
\usepackage{microtype}      
\usepackage{amsthm,amsmath,amssymb,epsfig,color,float,graphicx,verbatim, enumitem}
\usepackage{algpseudocode}
\usepackage{algorithm}

\newtheorem{theorem}{Theorem}[section]

\newtheorem{conj}{Conjecture}

\newtheorem{lemma}[theorem]{Lemma}
\newtheorem{informal theorem}[theorem]{Theorem (informal statement)}

\newtheorem{proposition}[theorem]{Proposition}

\newtheorem{claim}[theorem]{Claim}
\newtheorem{fact}[theorem]{Fact}

\newtheorem{remark}[theorem]{Remark}

\theoremstyle{definition}
\newtheorem{definition}[theorem]{Definition}

\newcommand{\eqdef}{\stackrel{{\mathrm {\footnotesize def}}}{=}}

\newcommand{\chow}{\mathbf{Chow}}

\newcommand{\bx}{\mathbf{x}}
\newcommand{\by}{\mathbf{y}}
\newcommand{\bv}{\mathbf{v}}
\newcommand{\bu}{\mathbf{u}}
\newcommand{\bw}{\mathbf{w}}
\newcommand{\br}{\mathbf{r}}
\newcommand{\bp}{\mathbf{p}}
\newcommand{\bc}{\mathbf{c}}

\newcommand{\bM}{\mathbf{M}}
\newcommand{\D}{\mathcal{D}}
\newcommand{\B}{\mathbb{B}}
\newcommand{\Sp}{\mathbb{S}}
\newcommand{\LTF}{\mathcal{H}}
\newcommand{\err}{\mathrm{err}}
\newcommand{\pr}{\mathbf{Pr}}
\newcommand{\E}{\mathbf{E}}

\newcommand{\opt}{\mathrm{OPT}}
\newcommand{\A}{\mathcal{A}}

\newcommand{\cB}{\mathcal{B}}
\newcommand{\cD}{\D}
\newcommand{\cL}{\mathcal{L}}
\newcommand{\cQ}{\mathcal{Q}}
\newcommand{\supp}{{\mathsf{supp}}}
\newcommand{\val}{\text{val}}

\newcommand{\tr}{\mathrm{tr}}
\newcommand{\proj}{\mathrm{Proj}}
\newcommand{\argmin}{\mathrm{argmin}}

\newcommand{\fA}{\mathcal{A}}
\newcommand{\fB}{\mathcal{B}}

\newcommand{\sgn}{\mathrm{sign}}
\newcommand{\sign}{\mathrm{sign}}
\newcommand{\polylog}{\mathrm{polylog}}
\newcommand{\polyloglog}{\text{polyloglog}}

\newcommand{\R}{\mathbb{R}}

\newcommand{\Z}{\mathbb{Z}}

\newcommand{\eps}{\epsilon}

\newcommand{\poly}{\mathrm{poly}}

\newcommand{\littleint}{\mathop{\textstyle \int}}
\newcommand{\littlesum}{\mathop{\textstyle \sum}}

\newcommand{\be}{\mathbf{e}}

\newcommand{\wh}{\widehat}

\def\colorful{0}

\ifnum\colorful=1
\newcommand{\new}[1]{{\color{red} #1}}

\else
\newcommand{\new}[1]{{#1}}

\fi

\title{Nearly Tight Bounds for Robust Proper Learning of Halfspaces with a Margin}

\author{
  Ilias Diakonikolas\thanks{Supported by NSF Award CCF-1652862 (CAREER) and a Sloan Research Fellowship. Part of this 
work was performed at the Simons Institute for the Theory of Computing during the program on Foundations of Data Science.} \\
  University of Wisconsin, Madison\\
  {\tt ilias@cs.wisc.edu} \\
  \and
  Daniel M. Kane\thanks{Supported by NSF Award CCF-1553288 (CAREER) and a Sloan Research Fellowship.} \\
  University of California, San Diego\\
  {\tt dakane@cs.ucsd.edu} \\
  \and
  Pasin Manurangsi\thanks{Now at Google Research, Mountain View.} \\
  University of California, Berkeley\\
 {\tt pasin@berkeley.edu}
}

\begin{document}

\maketitle

\begin{abstract}
We study the problem of {\em properly} learning large margin halfspaces in the agnostic PAC model. 
In more detail, we study the complexity of properly learning $d$-dimensional halfspaces
on the unit ball within misclassification error $\alpha \cdot \opt_{\gamma} + \eps$, where $\opt_{\gamma}$
is the optimal $\gamma$-margin error rate and $\alpha \geq 1$ is the approximation ratio.
We give learning algorithms and computational hardness results
for this problem, for all values of the approximation ratio $\alpha \geq 1$, 
that are nearly-matching for a range of parameters. 
Specifically, for the natural setting that $\alpha$ is any constant bigger 
than one, we provide an essentially tight complexity characterization.
On the positive side, we give an $\alpha = 1.01$-approximate proper learner 
that uses $O(1/(\eps^2\gamma^2))$ samples (which is optimal) and runs in time
$\poly(d/\eps) \cdot 2^{\tilde{O}(1/\gamma^2)}$. On the negative side,  
we show that {\em any} constant factor approximate proper learner has runtime 
$\poly(d/\eps) \cdot 2^{(1/\gamma)^{2-o(1)}}$, 
assuming the Exponential Time Hypothesis. 
\end{abstract}

\setcounter{page}{0}

\thispagestyle{empty}

\newpage

\section{Introduction} \label{sec:intro}

\subsection{Background and Problem Definition} \label{ssec:background}
Halfspaces are Boolean functions $h_{\bw}: \R^d \to \{ \pm 1\}$ of the form 
$h_{\bw}(\bx) = \sgn \left(\langle \bw, \bx \rangle \right)$, where $\bw \in \R^d$ is the associated weight vector. 
(The function $\sign: \R \to \{ \pm 1\}$ is defined as $\sgn(u)=1$ if $u \geq 0$ and $\sgn(u)=-1$ otherwise.)
The problem of learning an unknown halfspace with a margin condition (in the sense that no example
is allowed to lie too close to the separating hyperplane) is as old as the field of machine learning --- starting 
with Rosenblatt's Perceptron algorithm~\cite{Rosenblatt:58} --- and has arguably been one of 
the most influential problems in the development of the field, with techniques such as SVMs~\cite{Vapnik:98}
and AdaBoost~\cite{FreundSchapire:97} coming out of its study.


In this paper, we study the problem of learning $\gamma$-margin halfspaces in the 
{\em agnostic} PAC model~\cite{Haussler:92, KSS:94}. Specifically, 
there is an unknown distribution $\D$ on $\B_d \times \{ \pm 1\}$, 
where $\B_d$ is the unit ball on $\R^d$, and the learning 
algorithm $\mathcal{A}$ is given as input a training set $S = \{(\bx^{(i)}, y^{(i)}) \}_{i=1}^m$ of i.i.d. samples
drawn from $\D$. The goal of $\mathcal{A}$ is to output a hypothesis whose  
error rate is competitive with the $\gamma$-margin error rate of the optimal halfspace.
In more detail, the {\em error rate} (misclassification error) of a hypothesis $h: \R^d \to \{\pm 1\}$ (with respect to $\D$) is 
$\err_{0-1}^{\D}(h) \eqdef \pr_{(\bx, y) \sim \D}[h(\bx) \neq y]$.
For $\gamma \in (0, 1)$, the {\em $\gamma$-margin error rate} of a halfspace 
$h_{\bw}(\bx)$ with $\|\bw\|_2 \leq 1$ is 
$\err^{\D}_{\gamma}(\bw) \eqdef \pr_{(\bx, y) \sim \D} \left[y \langle \bw, x \rangle \leq \gamma \right]$.
We denote by $\opt_{\gamma}^{\D} \eqdef \min_{\|\bw\|_2 \leq 1} \err^{\D}_{\gamma}(\bw)$
the minimum $\gamma$-margin error rate achievable by any halfspace. 
We say that $\A$ is an {\em $\alpha$-agnostic learner}, $\alpha \geq 1$, 
if it outputs a hypothesis $h$ that with probability at least $1-\tau$ satisfies
$\err_{0-1}^{\D}(h) \leq \alpha \cdot \opt_{\gamma}^{\D} +\eps$. (For $\alpha$ = 1, we obtain
the standard notion of agnostic learning.) If the hypothesis $h$ is itself
a halfspace, we say that the learning algorithm is {\em proper}. This work 
focuses on proper learning algorithms.


\subsection{Related and Prior Work} \label{ssec:related-work}
In this section, we summarize the prior work that is directly related to the results of this paper.
First, we note that the sample complexity of our learning problem (ignoring computational considerations) 
is well-understood. In particular, the ERM that minimizes the number of {\em $\gamma$-margin errors} 
over the training set (subject to a norm constraint) is known to be an agnostic learner ($\alpha = 1$), 
assuming the sample size is $\Omega(\log(1/\tau)/(\eps^2\gamma^2))$. 
Specifically, $\Theta(\log(1/\tau)/(\eps^2\gamma^2))$ samples\footnote{To avoid clutter in the expressions,
we will henceforth assume that the failure probability $\tau = 1/10$. Recall that one can always boost
the confidence probability with an $O(\log(1/\tau))$ multiplicative overhead in the sample complexity.} 
are known to be sufficient and necessary 
for this learning problem (see, e.g.,~\cite{BartlettM02, McAllester03}). 
In the realizable case ($\opt_\gamma^{\D} = 0$), i.e., if the data is linearly separable with margin $\gamma$, 
the ERM rule above can be implemented in $\poly(d, 1/\eps, 1/\gamma)$ time using the Perceptron algorithm. 
The non-realizable setting ($\opt_\gamma^{\D} >0$) is much more challenging computationally.

The agnostic version of our problem ($\alpha=1$) was first considered in \cite{BenDavidS00}, 
who gave a {\em proper} learning algorithm with 
runtime $\poly(d) \cdot (1/\eps)^{\tilde{O}(1/\gamma^2)}$.
It was also shown in \cite{BenDavidS00} that agnostic proper learning 
with runtime $\poly(d, 1/\eps, 1/\gamma)$ is NP-hard. A question left open by their work 
was characterizing the computational complexity of proper learning as a function of $1/\gamma$. 

Subsequent works focused on {\em improper} learning.
The $\alpha=1$ case was studied in~\cite{SSS09, SSS10} who gave a
learning algorithm with sample complexity $\poly(1/\eps) \cdot 2^{\tilde{O}(1/\gamma)}$
-- i.e., {\em exponential} in $1/\gamma$ --  
and computational complexity $\poly(d/\eps) \cdot 2^{\tilde{O}(1/\gamma)}$. 
The increased sample complexity
is inherent in their approach, as their algorithm works by solving a convex program 
over an expanded feature space. 
\cite{BirnbaumS12} gave an $\alpha$-agnostic learning 
algorithm for all $\alpha \geq 1$ with sample complexity $\poly(1/\eps) \cdot 2^{\tilde{O}(1/(\alpha \gamma))}$ 
and computational complexity $\poly(d/\eps) \cdot 2^{\tilde{O}(1/(\alpha \gamma))}$. 
(We note that the Perceptron algorithm is known to achieve $\alpha = 1/\gamma$~\cite{Servedio:01lnh}. 
Prior to \cite{BirnbaumS12}, \cite{LS:11malicious} gave a $\poly(d, 1/\eps, 1/\gamma)$ time algorithm achieving 
$\alpha = \Theta ((1/\gamma)/\sqrt{\log(1/\gamma)})$.)
\cite{BirnbaumS12} posed as an open question whether their upper bounds 
for improper learning can be achieved with a proper learner. 

\new{A related line of work~\cite{KLS09, ABL17, DKKLMS16, LaiRV16, DKK+17, DKKLMS18-soda, 
DKS18a, KlivansKM18, DKS19, DKK+19-sever} has given polynomial time robust estimators for a range of learning tasks.
Specifically,~\cite{KLS09, ABL17, DKS18a, DKK+19-sever}
obtained efficient PAC learning algorithms for halfspaces with malicious noise~\cite{Valiant:85short, keali93}, under 
the assumption that the uncorrupted data comes from a ``tame'' distribution, e.g., Gaussian or isotropic log-concave. 
It should be noted that the class of $\gamma$-margin distributions considered in this work
is significantly broader and can be far from satisfying
the structural properties required in the aforementioned works.

A growing body of theoretical work has focused on \emph{adversarially robust learning} (e.g.,~\cite{BubeckLPR19,MontasserHS19,DegwekarNV19,Nakkiran19}). In adversarially robust learning, 
the learner seeks to output a hypothesis with small \emph{$\gamma$-robust misclassification error}, 
which for a hypothesis $h$ and a norm $\|\cdot\|$ is typically defined as 
$\pr_{(\bx, y) \sim \cD}[\exists \bx' \textrm{ with } \|\bx' - \bx\| \leq \gamma \textrm{ s.t. } h(\bx') \ne y]$. 
Notice that when $h$ is a halfspace and $\|\cdot\|$ is the Euclidean norm, 
the $\gamma$-robust misclassification error coincides with the $\gamma$-margin error in our context. 
(It should be noted that most of the literature on adversarially robust learning focuses on the $\ell_{\infty}$-norm.)
However, the objectives of the two learning settings are slightly different: in adversarially robust learning, 
the learner would like to output a hypothesis with small $\gamma$-robust misclassification error, 
whereas in our context the learner only has to output a hypothesis with small zero-one misclassification error. Nonetheless, 
as we point out in Remark~\ref{remark:adv-learning}, our algorithms can be adapted to provide 
guarantees in line with the adversarially robust setting as well.

Finally, in the distribution-independent agnostic setting without margin assumptions, there is compelling complexity-theoretic 
evidence that even weak learning of halfspaces is computationally intractable~\cite{GR:06, FGK+:06short, DOSW:11, Daniely16, BhattacharyyaGS18}.
}

\subsection{Our Contributions} \label{ssec:our-results}
We study the complexity of {\em proper} $\alpha$-agnostic learning of $\gamma$-margin
halfspaces on the unit ball. Our main result nearly characterizes the complexity of constant factor
approximation to this problem:

\begin{theorem} \label{thm:constant-factor-bounds}
There is an algorithm that uses $O(1/(\eps^2\gamma^2))$ samples, runs in time
$\poly(d/\eps) \cdot 2^{\tilde{O}(1/\gamma^2)}$ and is an $\alpha = 1.01$-agnostic proper learner
\new{for $\gamma$-margin halfspaces} with confidence probability $9/10$. 
Moreover, assuming the Randomized Exponential Time Hypothesis,
any proper learning algorithm that achieves \new{any} constant factor approximation has runtime
$\poly(d/\eps) \cdot \Omega(2^{(1/\gamma)^{2 - o(1)}})$.
\end{theorem}

The reader is referred to Theorems~\ref{thm:constant-factor-alg} and~\ref{thm:run-time} for detailed
statements of the upper and lower bound respectively.
A few remarks are in order: First, we note that the approximation ratio of $1.01$ in the above theorem statement 
is not inherent. Our algorithm achieves $\alpha = 1+\delta$, for any $\delta>0$, 
with runtime $\poly(d/\eps) \cdot 2^{\tilde{O}(1/(\delta \gamma^2))}$.
The runtime of our algorithm significantly improves the runtime of the best 
known agnostic proper learner~\cite{BenDavidS00}, achieving
fixed polynomial dependence on $1/\eps$, independent of $\gamma$. 
This gain in runtime comes at the expense of losing a small constant factor in the error guarantee.
It is natural to ask whether there exists an $1$-agnostic proper learner
matching the runtime of our Theorem~\ref{thm:constant-factor-bounds}. 
In Theorem~\ref{thm:param}, we establish a computational hardness result implying 
that such an improvement is unlikely. 

The runtime dependence of our algorithm scales as $2^{\tilde{O}(1/\gamma^2)}$ (which is nearly
best possible for proper learners), as opposed to $2^{\tilde{O}(1/\gamma)}$ 
in the best known improper learning algorithms~\cite{SSS09, BirnbaumS12}. 
In addition to the interpretability of proper learning, we note that the sample complexity of our algorithm
is quadratic in $1/\gamma$ (which is information-theoretically optimal), as opposed to exponential for known
improper learners. Moreover, for moderate values of $\gamma$, 
our algorithm may be faster than known improper learners, as it only uses
spectral methods and ERM, as opposed to convex optimization.
Finally, we note that the lower bound part of Theorem~\ref{thm:constant-factor-bounds} 
implies a computational separation between proper and improper learning for our problem.

In addition, we explore the complexity of $\alpha$-agnostic learning 
for large $\alpha>1$. The following theorem summarizes 
our results in this setting:

\begin{theorem} \label{thm:alpha-factor-bounds}
There is an algorithm that uses $\tilde O(1/(\eps^2\gamma^2))$ samples, runs in time
$\poly(d) \cdot (1/\eps)^{\tilde{O}(1/(\alpha \gamma)^2)}$ and is an 
$\alpha$-agnostic proper learner \new{for $\gamma$-margin halfspaces} with confidence probability $9/10$. 
Moreover, assuming NP $\ne$ RP and the Sliding Scale Conjecture, there exists an  
absolute constant $c > 0$, such that no $(1/\gamma)^c$-agnostic proper learner 
runs in $\poly(d,1/\varepsilon,1/\gamma)$ time.
\end{theorem}

The reader is referred to Theorem~\ref{alphaTheorem} for the upper bound
and Theorem~\ref{thm:inapx} for the lower bound.
In summary, we give an $\alpha$-agnostic proper learning algorithm with runtime 
exponential in $1/(\alpha\gamma)^2$, as opposed to $1/\gamma^2$, and we show that 
achieving $\alpha = (1/\gamma)^{\Omega(1)}$ is computationally hard. 
(Assuming only NP $\ne$ RP, we can rule out polynomial time 
$\alpha$-agnostic proper learning for $\alpha = (1/\gamma)^{\frac{1}{\polyloglog(1/\gamma)}}$.)

\new{
\begin{remark} \label{remark:adv-learning}
{\em While not stated explicitly in the subsequent analysis, our algorithms (with a slight modification to the associated 
constant factors) not only give a halfspace $\bw^{\ast}$ with zero-one loss at most $\alpha \cdot \opt_{\gamma}^{\D} +\eps$, 
but this guarantee holds for the $0.99\gamma$-margin error\footnote{Here the constant $0.99$ can be replaced 
by any constant less than one, with an appropriate increase to the algorithm's running time.} of $\bw^{\ast}$ as well. 
Thus, our learning algorithms also work in the adversarially robust setting (under the Euclidean norm) 
with a small loss in the ``robustness parameter'' (margin) from the one used to compute the optimum (i.e., $\gamma$) 
to the one used to measure the error of the output hypothesis (i.e., $0.99\gamma$).}
\end{remark}
}

\subsection{Our Techniques} \label{ssec:techniques}

\paragraph{Overview of Algorithms.}
For the sake of this intuitive explanation, we provide 
an overview of our algorithms when the underlying distribution 
$\D$ is explicitly known. The finite sample analysis of our algorithms follows
from standard generalization bounds (see Section~\ref{sec:alg}).

Our constant factor approximation algorithm relies on the following observation: 
Let $\bw^{\ast}$ be the optimal weight vector.
The assumption that $|\langle \bw^{\ast}, \bx \rangle |$ is large for almost all $\bx$ (by the margin property)
implies a relatively strong condition on $\bw^{\ast}$, which will allow us to find 
a relatively small search space containing a near-optimal solution. 
A first idea is to consider the matrix $\bM = \E_{(\bx, y) \sim \D}[\bx \bx^T]$ 
and note that ${\bw^{\ast}}^T \bM \bw^{\ast} = \Omega(\gamma^2)$. 
This in turn implies that $\bw^{\ast}$ has a large component on the subspace 
spanned by the largest $O(1/(\eps\gamma^2))$ eigenvalues of $\bM$.
This idea suggests a basic algorithm that computes a net over 
unit-norm weight vectors on this subspace and outputs the best answer.
\new{This basic algorithm has runtime $\poly(d) \cdot 2^{\tilde O(1/(\eps\gamma^2))}$
and is analyzed in Section~\ref{ssec:alg-basic}.}

To obtain our $\poly(d/\eps) \cdot 2^{\tilde O(1/\gamma^2)}$ time constant factor approximation 
algorithm (establishing the upper bound part of Theorem~\ref{thm:constant-factor-bounds}), 
we use a refinement of the above idea.
Instead of trying to guess the projection of $\bw^{\ast}$ onto the space of
large eigenvectors {\em all at once}, we will do so in stages. In
particular, it is not hard to see that $\bw^{\ast}$ has a non-trivial
projection onto the subspace spanned by the top $O(1/\gamma^2)$
eigenvalues of $\bM$. If we guess this projection, we will have some
approximation to $\bw^{\ast}$, but unfortunately not a sufficiently good one. 
However, we note that the difference between $\bw^{\ast}$ and our current hypothesis $\bw$ 
will have a large average squared inner product with the misclassified points.
This suggests an iterative algorithm that in the $i$-th iteration 
considers the second moment matrix $\bM^{(i)}$ of the points 
not correctly classified by the current hypothesis $\sgn(\langle \bw^{(i)}, \bx \rangle)$, 
guesses a vector $\bu$ in the space spanned by the top few eigenvalues of $\bM^{(i)}$, 
and sets $\bw^{(i+1)} = \bu + \bw^{(i)}$. This procedure can be shown to produce
a candidate set of weights with cardinality $2^{\tilde O (1/\gamma^2)}$  
one of which has the desired misclassification error. \new{This algorithm and its analysis
are given in Section~\ref{ssec:alg-main}.}

Our general $\alpha$-agnostic algorithm (upper bound in Theorem~\ref{thm:alpha-factor-bounds}) relies on approximating
the {\em Chow parameters} of the target halfspace $f_{\bw^{\ast}}$, i.e., 
the $d$ numbers $\E[f_{\bw^{\ast}}(\bx) \bx_i]$, $i \in [d]$.
A classical result~\cite{Chow:61} shows that the exact values of the 
Chow parameters of a halfspace (over any distribution) uniquely define the halfspace. 
Although this fact is not very useful under an arbitrary distribution, 
the margin assumption implies a strong {\em approximate identifiability} result 
(Lemma~\ref{lem:chow-vs-dist}).  Combining this with an algorithm of~\cite{DeDFS14}, 
we can efficiently compute an approximation to the halfspace $f_{\bw^{\ast}}$ given an approximation to its Chow parameters. 
In particular, if we can approximate the Chow parameters to $\ell_2$-error
\new{$\nu \cdot \gamma$}, we can approximate $f_{\bw^{\ast}}$ within error \new{$\opt_{\gamma}^{\D}+\nu$}. 

A naive approach to approximate the Chow parameters would be via the empirical
Chow parameters, namely $\E_{(\bx, y) \sim \D}[y \bx]$.  In the realizable case, this quantity 
indeed corresponds to the vector of Chow parameters. Unfortunately however, this method does not
work in the agnostic case and it can introduce an error of $\omega(\opt_{\gamma}^{\D})$.
To overcome this obstacle, we note that in order for a small fraction of errors to introduce
a large error in the empirical Chow parameters, it must be the case
that there is some direction $\bw$ in which many of these erroneous points
introduce a large error. If we can guess some error that correlates well with $\bw$ 
and also guess the correct projection of our Chow parameters onto this vector, 
we can correct a decent fraction of the error between the empirical and true Chow parameters. 
We show that making the correct guesses $\tilde O (1/(\gamma \alpha)^2)$ times, 
we can reduce the empirical error sufficiently so that it can be used to find an accurate
hypothesis. Once again, we can compute a hypothesis for each sequence of
guesses and return the best one. 
\new{See Section~\ref{ssec:alg-bicrit} for a detailed analysis.}

\paragraph{Overview of Computational Lower Bounds.} 
Our hardness results are shown via two reductions. 
These reductions take as input an instance of a computationally hard problem and 
produce a distribution $\D$ on $\B_d \times \{\pm 1\}$. 
If the starting instance is a YES instance of the original problem, 
then $\opt_{\gamma}^{\D}$ is small for an appropriate value of $\gamma$. 
On the other hand, if the starting instance is a NO instance of the original problem, 
then $\opt_{0-1}^{\D}$ is large\footnote{We use $\opt_{0-1}^{\D} \eqdef \min_{\bw \in \R^d} \err^{\D}_{0-1}(\bw)$ 
to denote the minimum error rate achievable by any halfspace.}. As a result, if there is a ``too fast'' ($\alpha$-)agnostic 
proper learner for $\gamma$-margin halfspaces, then we would also get a ``too fast'' 
algorithm for the original problem as well, which would violate the corresponding 
complexity assumption.

To understand the margin parameter $\gamma$ we can achieve, we need to first understand the problems we start with. 
For our reductions, the original problems can be viewed in the following form: 
select $k$ items from $v_1, \dots, v_N$ that satisfy certain ``local constraints''. 
For instance, in our first construction, the reduction is from the $k$-Clique problem: 
Given a graph $G$ and an integer $k$, the goal is to determine 
whether $G$ contains a $k$-clique as a subgraph. 
For this problem, $v_1, \dots, v_N$ correspond to the vertices of $G$ and the ``local'' constraints 
are that every pair of selected vertices induces an edge.

Roughly speaking, our reduction produces a distribution $\D$ on $\B_d \times \{\pm 1\}$ in dimension $d = N$, 
with the $i$-th dimension corresponding to $v_i$. The ``ideal'' solution in the YES case is to set 
$\bw_i = \frac{1}{\sqrt{k}}$ iff $v_i$ is selected and set $\bw_i = 0$ otherwise. 
In our reductions, the local constraints are expressed using ``sparse'' sample vectors 
(i.e., vectors with only a constant number of non-zero coordinates all having the same magnitude). 
For example, in the case of $k$-Clique, the constraints can be expressed as follows: 
For every non-edge $(i, j)$, we must have $\left(\frac{1}{\sqrt{2}}\be^i + \frac{1}{\sqrt{2}}\be^j\right) \cdot \bw \leq \frac{1}{\sqrt{2k}}$, 
where $\be^i$ and $\be^j$ denote the $i$-th and $j$-th vectors in the standard basis. 
A main step in both of our proofs is to show that the reduction still works even when we ``shift'' the right hand side 
by a small multiple of $\frac{1}{\sqrt{k}}$. For instance, in the case of $k$-Clique, it is possible to show that, 
even if we replace $\frac{1}{\sqrt{2k}}$ with, say, $\frac{0.99}{\sqrt{2k}}$, the correctness of the construction 
remains, and we also get the added benefit that now the constraints are satisfied with a margin of 
$\gamma = \Theta(\frac{1}{\sqrt{k}})$ for our ideal solution in the YES case.

In the case of $k$-Clique, the above idea yields a reduction to 1-agnostic learning $\gamma$-margin halfspaces 
with margin $\gamma = \Theta(\frac{1}{\sqrt{k}})$, where the dimension $d$ is $N$ (and $\varepsilon = \frac{1}{\poly(N)}$). 
As a result, if there is an $f(\frac{1}{\gamma})\poly(d,\frac{1}{\varepsilon})$-time algorithm for the latter for some function $f$, 
then there also exists a $g(k)\poly(N)$-time algorithm for $k$-Clique for some function $g$. The latter statement is 
considered unlikely, as it would break a widely-believed hypothesis in the area of parameterized complexity. 

Ruling out $\alpha$-agnostic learners, for $\alpha>1$, is slightly more complicated, 
since we need to produce the ``gap'' of $\alpha$ between $\opt_{\gamma}^{\D}$ in the YES case 
and $\opt_{0-1}^{\D}$ in the NO case. To create such a gap, we appeal to the PCP Theorem~\cite{AroraS98,AroraLMSS98}, 
which can be thought of as an NP-hardness proof of the following ``gap version'' of 3SAT: given a 3CNF formula as input, 
distinguish between the case that the formula is satisfiable and the case that the formula is not even 
$0.9$-satisfiable\footnote{In other words, for any assignment to the variables, at least $0.1$ fraction of the clauses are unsatisfied.}. 
Moreover, further strengthened versions of the PCP Theorem~\cite{Dinur07,MR10} actually implies that this Gap-3SAT problem 
cannot even be solved in time $O(2^{n^{0.999}})$, where $n$ denotes the number of variables in the formula, 
assuming the Exponential Time Hypothesis (ETH)\footnote{ETH states that the \emph{exact} version of 3SAT cannot be solved in $2^{o(n)}$ time.}. Once again, (Gap-)3SAT can be viewed in the form of ``item selection with local constraints''. However, the number 
of selected items $k$ is now equal to $n$, the number of variables of the formula. With a similar line of reasoning as above, 
the margin we get is now $\gamma = \Theta(\frac{1}{\sqrt{k}}) = \Theta(\frac{1}{\sqrt{n}})$. As a result, if there is, say, 
a $2^{(1/\gamma)^{1.99}}\poly(d,\frac{1}{\varepsilon})$-time $\alpha$-agnostic proper learner for $\gamma$-margin 
halfspaces (for an appropriate $\alpha$), then there is an $O(2^{n^{0.995}})$-time algorithm for Gap-3SAT, which would violate ETH.

Unfortunately, the above described idea only gives the ``gap'' $\alpha$ that is only slightly larger than $1$, 
because the gap that we start with in the Gap-3SAT problem is already pretty small. To achieve larger gaps, 
our actual reduction starts from a generalization of 3SAT, called constraint satisfaction problems (CSPs), 
whose gap problems are hard even for very large gap. 
This concludes the outline of the main intuitions in our reductions.
\new{The detailed proofs are given in Section~\ref{sec:lb}.}

\subsection{Preliminaries} \label{ssec:prelims}
For $n \in \Z_+$, we denote $[n] \eqdef \{1, \ldots, n\}$.
We will use small boldface characters for vectors and capital
boldface characters for matrices. For a vector $\bx \in \R^d$, 
and $i \in [d]$, $\bx_i$ denotes the $i$-th coordinate of $\bx$, and 
$\|\bx\|_2 \eqdef (\littlesum_{i=1}^d \bx_i^2)^{1/2}$ denotes the $\ell_2$-norm
of $\bx$. We will use $\langle \bx, \by \rangle$ for the inner product between $\bx, \by \in \R^d$. 
For a matrix $\bM \in \R^{d \times d}$, we will denote by $\|\bM\|_2$ its spectral norm and by
$\tr(\bM)$ its trace.
Let $\B_d = \{ \bx \in \R^d: \|\bx\|_2 \leq 1 \}$ be the unit ball and 
$\Sp_{d-1} = \{ \bx \in \R^d: \|\bx\|_2 = 1 \}$ be the unit sphere in $\R^d$.

An origin-centered halfspace is a Boolean-valued function $h_{\bw}: \R^d \to \{\pm 1\}$ 
of the form $h_{\bw}(\bx) = \sgn \left(\langle \bw, \bx \rangle \right)$,
where $\bw \in \R^d$. (Note that we may assume w.l.o.g. that $\|\bw\|_2 =1$.) 
Let $\LTF_{d} = \left\{ h_{\bw}(\bx) = \sgn \left(\langle \bw, \bx \rangle \right), \bw \in \R^d \right\}$ 
denote the class of all origin-centered halfspaces on $\R^d$.
Finally, we use $\be^i$ to denote the $i$-th standard basis vector, 
i.e., the vector whose $i$-th coordinate is one and the remaining coordinates are zero.

\section{Efficient Proper Agnostic Learning of Halfspaces with a Margin} \label{sec:alg}

\subsection{Warm-Up: Basic Algorithm} \label{ssec:alg-basic}
In this subsection, we present a basic algorithm that achieves $\alpha = 1$
and whose runtime is $\poly(d) 2^{\tilde{O}(1/(\eps\gamma^2))}$.
Despite its slow runtime, this algorithm serves as a warm-up for our more
sophisticated constant factor approximation algorithm in the next subsection.

We start by establishing a basic structural property of this setting which motivates our basic algorithm.
We start with the following simple claim:

\begin{claim} \label{claim:M-spectral-norm}
Let $\bM^{\D} = \E_{{(\bx,y)\sim \D}} [\bx \bx^T]$ 
and $\bw^{\ast}$ be a unit vector such that
$\err^{\D}_{\gamma}(\bw^{\ast})  \leq \opt_{\gamma}^{\D} \leq 1/2$.
Then, we have that $\| \bM^{\D}\|_2 \geq {\bw^{\ast}}^T \bM^{\D} \bw^{\ast} \geq \gamma^2/2$.
\end{claim}
\begin{proof}
By assumption,
$\pr_{(\bx,y)\sim \D}[ \left| \langle \bw^{\ast} , \bx  \rangle \right| \geq \gamma]  \geq 1/2$,
which implies that $\E_{(\bx,y)\sim \D}[ \left(\langle \bw^{\ast} , \bx \rangle \right)^2] \geq \gamma^2/2$.
The claim follows from the fact that $\bv^{T} \bM^{\D} \bv = \E_{(\bx,y)\sim \D}[ \left(\langle \bv , \bx \rangle \right)^2]$,
for any $\bv \in \R^d$, and the definition of the spectral norm.
\end{proof}

Claim~\ref{claim:M-spectral-norm} allows us to obtain an approximation to the optimal halfspace by projecting
on the space of large eigenvalues of $\bM^{\D}$. We will need the following terminology:
For $\delta>0$, let $V^{\D}_{\geq \delta}$ be the space spanned by the eigenvalues of $\bM^{\D}$
with magnitude at least $\delta$ and $V^{\D}_{< \delta}$ be its complement.
Let $\proj_V(\bv)$ denote the projection operator of vector $\bv$ on subspace $V$.
Then, we have the following:

\begin{lemma} \label{lem:proj-approx}
Let $\delta>0$ and $\bw' = \proj_{V^{\D}_{\geq \delta}} (\bw^{\ast})$.
Then, we have that $\err^{\D}_{\gamma/2}(\bw') \leq \err^{\D}_{\gamma}(\bw^{\ast}) + 4\delta/\gamma^2.$
\end{lemma}
\begin{proof}
Let $\bw^{\ast} = \bw'+ \bw''$, where $\bw'' = \proj_{V^{\D}_{< \delta}} (\bw^{\ast})$.
Observe that for any $(\bx, y)$, if $y \langle \bw', \bx \rangle \leq \gamma/2$ then
$y \langle \bw^{\ast}, \bx \rangle \leq \gamma$, unless $|\langle \bw'', \bx \rangle| \geq \gamma/2$.
Hence,
$\err^{\D}_{\gamma/2}(\bw') \leq \err^{\D}_{\gamma}(\bw^{\ast}) + \pr_{(\bx, y) \sim \D}[|\langle \bw'' , \bx \rangle| \geq \gamma/2]$.
By definition of $\bw''$ and $\bM^{\D}$, we have that $\E_{{(\bx, y) \sim \D}}[(\langle \bw'' , \bx \rangle)^2] \leq \delta$.
By Markov's inequality, we thus obtain $\pr_{(\bx, y) \sim \D}[(\langle \bw'' , \bx \rangle)^2 \geq \gamma^2/4] \leq 4\delta/\gamma^2$,
completing the proof of the lemma.
\end{proof}

Motivated by Lemma~\ref{lem:proj-approx}, the idea is to enumerate over $V_{\geq \delta}^{\D}$,
for $\delta = \Theta(\eps \gamma^2)$,
and output a vector $\bv$ with smallest empirical $\gamma/2$-margin error.
To turn this into an actual algorithm, we work with a finite sample set
and enumerate over an appropriate cover of the space $V_{\geq \delta}^{{\D}}$.
The pseudocode is as follows:
\begin{algorithm}
\caption{Basic $1$-Agnostic Proper Learning Algorithm}
\begin{algorithmic}[1]
\State Draw a multiset $S = \{(\bx^{(i)}, y^{(i)}) \}$ of i.i.d. samples from $\D$, where $m = \Omega(\log(1/\tau)/(\eps^2\gamma^2))$.
\State 
Let $\wh{\D}_m$ be the empirical distribution on $S$.
\State Let $\bM^{\wh{\D}_m} = \E_{{(\bx,y)\sim \wh{\D}_m}} [\bx \bx^T]$.
\State Set $\delta = \eps \gamma^2/16$. Use SVD to find a basis of $V^{\wh{\D}_m}_{\geq \delta}$.
\State Compute a $\delta/2$-cover, $C_{\delta/2}$, in $\ell_2$-norm, of $V^{\wh{\D}_m}_{\geq \delta} \cap \Sp_{d-1}$.
\State Let $\bv \in \argmin_{\bw \in C_{\delta/2}} \err_{\gamma/4}^{\wh{D}_m}(\bw)$.
\State \Return $h_{\bv}(\bx) = \sgn(\langle \bv, \bx \rangle)$.
\end{algorithmic}
\end{algorithm}

First, we analyze the runtime of our algorithm.
The SVD of $\bM^{\wh{\D}_m}$ can be computed in $\poly(d/\delta)$ time.
Note that $V^{\wh{\D}_m}_{\geq \delta}$ has dimension at most $1/\delta$.
This follows from the fact that $\bM^{\wh{\D}_m}$ is PSD and its trace is
$\sum_{i=1}^d \lambda_i = \tr(\bM^{\wh{\D}_m}) = \E_{{(\bx,y)\sim \wh{\D}_m}}[\tr(\bx \bx^T)] \leq 1$,
where we used that $\|\bx\|_2 \leq 1$ with probability $1$ over  $\wh{\D}_m$.
Therefore, the unit sphere of  $V^{\wh{\D}_m}_{\geq \delta}$ has a $\delta/2$-cover $C_{\delta/2}$ of size
$(2/\delta)^{O(1/\delta)} = 2^{\tilde{O}(1/(\eps \gamma^2))}$ that can be computed in output polynomial time.

We now prove correctness. The main idea is to apply Lemma~\ref{lem:proj-approx}
for the empirical distribution $\wh{\D}_m$ combined with 
\new{the following statistical bound:}
\begin{fact}[\cite{BartlettM02, McAllester03}]\label{fact:erm-margin}
Let $S = \{(\bx^{(i)}, y^{(i)}) \}_{i=1}^m$ be a multiset of i.i.d. samples from $\D$, where $m = \Omega(\log(1/\tau)/(\eps^2\gamma^2))$,
and  $\wh{\D}_m$ be the empirical distribution on $S$. Then with probability at least $1-\tau$ over $S$,
simultaneously for all unit vectors $\bw$ and margins $\gamma>0$, if $h_{\bw}(\bx) = \sign(\langle \bw, \bx \rangle)$,
we have that $\err^{\D}_{0-1} (h_{\bw}) \leq \err_{\gamma}^{\wh{\D}_m} (\bw) + \eps$.
\end{fact}

We proceed with the formal proof.
First, we claim that for
$m = \Omega(\log(1/\tau)/\eps^2)$, with probability at least $1-\tau/2$ over $S$, we have that
$\err_{\gamma}^{\wh{\D}_m}(\bw^{\ast}) \leq \err_{\gamma}^{\D} (\bw^{\ast}) +\eps/8$. To see this,
note that $\err_{\gamma}^{\wh{\D}_m} (\bw^{\ast})$ can be viewed as a sum of Bernoulli random variables
with expectation $\err_{\gamma}^{\D} (\bw^{\ast})$. Hence, the claim follows by a Chernoff bound.
By an argument similar to that of Lemma~\ref{lem:proj-approx}, we have that
$\err_{\gamma/4}^{\wh{D}_m}(\bv) \leq  \err_{\gamma/2}^{\wh{D}_m}(\bw') + \eps/2$.
Indeed, we can write $\bv = \bw' + \br$, where $\|\br\|_2 \leq \delta/2$, and follow the same argument.

In summary, we have the following sequence of inequalities:
\begin{eqnarray*}
\err_{\gamma/4}^{\wh{D}_m}(\bv) &\leq&  \err_{\gamma/2}^{\wh{D}_m}(\bw') + \eps/2
\leq  \err_{\gamma}^{\wh{D}_m}(\bw^{\ast}) + \eps/2 + \eps/4 \\
&\leq& \err_{\gamma}^{\D}(\bw^{\ast}) + \eps/2 + \eps/4 + \eps/8 \;,
\end{eqnarray*}
where the second inequality uses Lemma~\ref{lem:proj-approx} for $\wh{\D}_m$.
Finally, we use Fact~\ref{fact:erm-margin} for $\gamma/4$ and $\eps/8$ to obtain that
$\err^{\D}_{0-1} (h_{\bv}) \leq \err_{\gamma/4}^{\wh{D}_m}(\bv)+\eps/8 \leq \opt_{\gamma}^{\D} +\eps$.
The proof follows by a union bound.

\subsection{Main Algorithm: Near-Optimal Constant Factor Approximation} \label{ssec:alg-main}

In this section, we establish the following theorem, which gives the upper bound
part of Theorem~\ref{thm:constant-factor-bounds}:

\begin{theorem} \label{thm:constant-factor-alg}
Fix $0< \delta \leq 1$. There is an algorithm that uses $O(1/(\eps^2\gamma^2))$ samples, runs in time
$\poly(d/\eps) \cdot 2^{\tilde{O}(1/(\delta \gamma^2))}$ and is a $(1+\delta)$-agnostic proper learner 
\new{for $\gamma$-margin halfspaces} with confidence probability $9/10$.
\end{theorem}

Our algorithm in this section produces
a finite set of candidate weight vectors and outputs the one with the smallest empirical
$\gamma/2$-margin error. For the sake of this intuitive description,
we will assume that the algorithm knows the distribution $\D$ in question supported
on $\B_d \times \{ \pm 1\}$. By assumption, there is a unit vector $\bw^{\ast}$ so that
$\err_{\gamma}^{\D}(\bw^{\ast}) \leq \opt_{\gamma}^{\D}$.

We note that if a hypothesis $h_{\bw}$ defined by vector $\bw$ has $\gamma/2$-margin error
at least a $(1+\delta)\opt_{\gamma}^{\D}$, then there must be a large number of points 
correctly classified with $\gamma$-margin by $h_{\bw^\ast}$, 
but not correctly classified with $\gamma/2$-margin by $h_\bw$. 
For all of these points, we must have that $|\langle \bw^{\ast}-\bw,  \bx \rangle| \geq \gamma/2$.
This implies that the $\gamma/2$-margin-misclassified points of $h_\bw$ have a large covariance in the $\bw^{\ast}-\bw$ direction.
In particular, we have:
\begin{claim} \label{clm:spectral-norm-diff}
Let $\bw \in \R^d$ be such that $\err_{\gamma/2}^{\D}(\bw) > (1+\delta) \opt_{\gamma}^{\D}$.
Let $\D'$ be $\D$ conditioned on $y \langle \bw, \bx \rangle\leq \gamma/2$.
Let $\bM^{\D'} = \E_{(\bx,y)\sim \D'}[\bx \bx^T]$. Then
$(\bw^{\ast}-\bw)^T \bM^{\D'} (\bw^{\ast}-\bw) \geq \delta \gamma^2/8.$
\end{claim}
\begin{proof}
We claim that with probability at least $\delta/2$ over $(\bx, y) \sim \D'$
we have that $y \langle \bw, \bx \rangle \leq \gamma/2$
and $y \langle \bw^{\ast}, \bx \rangle \geq \gamma$. To see this, we first note that
$\pr_{(\bx, y) \sim \D'}[y \langle \bw, \bx \rangle >\gamma/2]=0$ holds by definition of $\D'$.
Hence, we have that
$$\pr_{(\bx, y) \sim \D'}[y \langle \bw^{\ast}, \bx \rangle \leq \gamma]
\leq \frac{\pr_{(\bx, y) \sim \D}[y \langle \bw^{\ast}, \bx \rangle \leq \gamma]}{\pr_{(\bx, y) \sim \D}[y \langle \bw, \bx \rangle\leq \gamma/2]}
< \frac{\opt_\gamma^{\D}}{(1+\delta) \opt_{\gamma}^{\D}} = \frac{1}{(1+\delta)} \;.$$
By a union bound, we obtain
$\pr_{(\bx, y) \sim \D'}[(y \langle \bw, \bx \rangle >\gamma/2) \cup (y \langle \bw^{\ast}, \bx \rangle \leq \gamma)] \leq \frac{1}{(1+\delta)}$.

Therefore, with probability at least $\delta/(1+\delta) \geq \delta/2$ (since $\delta \leq 1$)
over $(\bx, y) \sim \D'$ we have that $y \langle \bw^{\ast}-\bw, \bx\rangle \geq \gamma/2$,
which implies that $\langle \bw^{\ast}-\bw, \bx \rangle^2 \geq \gamma^2/4$. Thus,
$(\bw^{\ast}-\bw)^T \bM^{\D'} (\bw^{\ast}-\bw) =
\E_{(\bx,y)\sim \D'}[(\langle \bw^{\ast}-\bw, \bx \rangle)^2] \geq \delta \gamma^2/8$,
completing the proof.
\end{proof}

Claim~\ref{clm:spectral-norm-diff} says that $\bw^{\ast}-\bw$ has a large component
on the large eigenvalues of $\bM^{\D'}$. Building on this claim, we obtain the following result:
\begin{lemma} \label{lem:large-projection-k}
Let $\bw^{\ast},\bw,\bM^{\D'}$ be as in Claim~\ref{clm:spectral-norm-diff}. There exists $k \in \Z_+$
so that if $V_k$ is the span of the top $k$ eigenvectors of $\bM^{\D'}$,
we have that $\|\proj_{V_k}(\bw^{\ast}-\bw)\|_2^2 \geq k \delta \gamma^2/8$.
\end{lemma}
\begin{proof}
Note that the matrix $\bM^{\D'}$ is PSD and let $0> \lambda_{\max} = \lambda_1 \geq \lambda_2 \geq \ldots \geq \lambda_d \geq 0$
be its set of eigenvalues. We will denote by $V_{\geq t}$ the space spanned by the
eigenvectors of  $\bM^{\D'}$ corresponding to eigenvalues of magnitude at least $t$.
Let $d_t = \dim(V_{\geq t})$ be the dimension of $V_{\geq t}$,
i.e., the number of $i \in [d]$ with $\lambda_i \geq t$.
Since $\bx$ is supported on the unit ball, for $(\bx, y) \sim \D'$, we have that
$\tr(\bM^{\D'}) = \E_{(\bx, y) \sim \D'}[\tr(\bx \bx^T)] \leq 1$.
Since $\bM^{\D'}$ is PSD, we have that $\tr(\bM^{\D'}) = \littlesum_{i=1}^d \lambda_i$ and we can write
\begin{equation} \label{eqn:ena}
1\geq \tr(\bM^{\D'}) = \littlesum_{i=1}^d \lambda_i = \littlesum_{i=1}^d \littleint_0^{\lambda_i} 1 dt
=\littlesum_{i=1}^d \littleint_0^{\lambda_{\max}} \mathbf{1}_{\lambda_i \geq t} dt  = \littleint_0^{\lambda_{\max}} d_t dt,
\end{equation}
where the last equality follows by changing the order of the summation and the integration.
If the projection of $(\bw^{\ast}-\bw)$ onto the $i$-th eigenvector of $\bM^{\D'}$
has $\ell_2$-norm $a_i$, we have that
\begin{equation} \label{eqn:dyo}
\delta \gamma^2/8 \leq (\bw^{\ast}-\bw)^T \bM^{\D'} (\bw^{\ast}-\bw) = \littlesum_{i=1}^d \lambda_i a_i^2
=  \littlesum_{i=1}^d \littleint_0^{\lambda_{\max}} a_i^2 \mathbf{1}_{\lambda_i \geq t} dt
=\littleint_0^{\lambda_{\max}} \|\proj_{V_{\geq t}}(\bw^{\ast}-\bw)\|_2^2 dt,
\end{equation}
where the first inequality uses Claim~\ref{clm:spectral-norm-diff}, the first equality
follows by the Pythagorean theorem, and the last equality follows
by changing the order of the summation and the integration.

Combining \eqref{eqn:ena} and \eqref{eqn:dyo}, we obtain
$\littleint_0^{\lambda_{\max}} \|\proj_{V_{\geq t}}(\bw^{\ast}-\bw)\|_2^2 dt \geq
(\delta \gamma^2/8) \littleint_0^{\lambda_{\max}} d_t dt $.
By an averaging argument, there exists $0\leq t \leq \lambda_{\max}$ such that
$\|\proj_{V_{\geq t}}(\bw^{\ast}-\bw)\|_2^2 \geq (\delta \gamma^2/8) d_t$.
Letting $k=d_t$ and noting that $V_{\geq t} = V_k$ completes the proof.
\end{proof}

Lemma~\ref{lem:large-projection-k} suggests a method for producing an approximation to $\bw^{\ast}$,
or more precisely a vector that produces empirical $\gamma/2$-margin error at most $(1+\delta)\opt_{\gamma}^{\D}$.
We start by describing a non-deterministic procedure, which we will then turn into an actual
algorithm.

The method proceeds in a sequence of stages.
At stage $i$, we have a hypothesis weight vector $\bw^{(i)}$.
(At stage $i=0$, we start with $\bw^{(0)}=\mathbf{0}$.)
At any stage $i$, if $\err_{\gamma/2}^{\D} (\bw^{(i)}) \leq (1+\delta) \opt_{\gamma}^{\D}$,
then $\bw^{(i)}$ is a sufficient estimator. Otherwise, we consider the matrix $\bM^{(i)} =  \E_{(\bx,y)\sim \D^{(i)}} [\bx\bx^T]$,
where $\D^{(i)}$ is $\D$ conditioned on $y \langle \bw^{(i)},  \bx \rangle \leq \gamma/2$.
By Lemma~\ref{lem:large-projection-k}, we know that for some positive integer value $k^{(i)}$,
we have that the projection of $\bw^{\ast}-\bw^{(i)}$ onto $V_{k^{(i)}}$
has squared norm at least $\delta k^{(i)} \gamma^2/8$.

Let $\bp^{(i)}$ be this projection. We set $\bw^{(i+1)}=\bw^{(i)}+\bp^{(i)}$.
Since the projection of $\bw^{\ast}-\bw^{(i)}$ and its complement are orthogonal, we have
\begin{equation} \label{eqn:norm-decreases}
\|\bw^{\ast}-\bw^{(i+1)}\|_2^2 = \|\bw^{\ast}-\bw^{(i)}\|_2^2 - \|\bp^{(i)}\|_2^2
\leq \|\bw^{\ast}-\bw^{(i)}\|_2^2 - \delta k^{(i)}\gamma^2/8 \;,
\end{equation}
where the inequality uses the fact that $\|\bp^{(i)} \|_2^2 \geq k^{(i)} \delta \gamma^2/8$
(as follows from Lemma~\ref{lem:large-projection-k}).
Let $s$ be the total number of stages. We can write
$$1 \geq \|\bw^{\ast}-\bw^{(0)}\|_2^2 - \|\bw^{\ast}-\bw^{(s)}\|_2^2 =
\littlesum_{i=0}^{s-1} \left(\|\bw^{\ast}-\bw^{(i)}\|_2^2 - \|\bw^{\ast}-\bw^{(i+1)}\|_2^2\right)
\geq  (\delta \gamma^2/8) \littlesum_{i=0}^{s-1} k^{(i)} \;,$$
where the first inequality uses that $\|\bw^{\ast}-\bw^{(0)}\|_2^2 =1$ and
$\|\bw^{\ast}-\bw^{(s)}\|_2^2 \geq 0$, the second notes the telescoping sum, and the third uses
\eqref{eqn:norm-decreases}. We thus have that $s \leq \littlesum_{i=0}^{s-1} k^{(i)} \leq 8/(\delta \gamma^2)$.
Therefore, the above procedure terminates after at most $8/(\delta \gamma^2)$ stages at some $\bw^{(s)}$
with $\err_{\gamma/2}^{\D} (\bw^{(s)})  \leq (1+\delta) \opt_{\gamma}^{\D}$.

We now describe how to turn the above procedure into an actual algorithm.
Our algorithm tries to simulate the above described procedure by making appropriate guesses.
In particular, we start by guessing a sequence of positive integers $k^{(i)}$ whose sum is at most
$8/(\delta \gamma^2)$.
This can be done in $2^{O(1/(\delta \gamma^2))}$ ways.
Next, given this sequence, our algorithm guesses the vectors $\bw^{(i)}$ over all $s$ stages in order.
In particular, given $\bw^{(i)}$, the algorithm computes the matrix $\bM^{(i)}$ and the subspace $V_{k^{(i)}}$,
and guesses the projection $\bp^{(i)} \in V_{k^{(i)}}$, which then gives $\bw^{(i+1)}$.
Of course, we cannot expect our algorithm to guess $\bp^{(i)}$ exactly (as there are infinitely many points in $V_{k^{(i)}}$),
but we can guess it to within $\ell_2$-error $\poly(\gamma)$, by taking an appropriate net.
This involves an additional guess of size $(1/\gamma)^{O(k^{(i)})}$ in each stage.
In total, our algorithm makes $2^{\tilde O(1/(\delta \gamma^2))}$ many different guesses.

We note that the sample version of our algorithm is essentially identical to the idealized
version described above, by replacing the distribution $\D$ by its empirical version and leveraging
\new{Fact~\ref{fact:erm-margin}.}

\noindent The pseudo-code is given in Algorithm~\ref{alg:opt-constant} below.

\medskip

\begin{algorithm}
\caption{\label{alg:opt-constant} Near-Optimal $(1+\delta)$-Agnostic Proper Learner}
\begin{algorithmic}[1]
\State Draw a multiset $S = \{(\bx^{(i)}, y^{(i)}) \}_{i=1}^m$ of i.i.d. samples from $\D$,
where $m = \Omega(\log(1/\tau)/(\eps^2\gamma^2))$.
\State 
Let $\wh{\D}_m$ be the empirical distribution on $S$.
\For{all sequences $k^{(0)}, k^{(1)},\ldots, k^{(s-1)}$ of positive integers
with sum at most $8/(\delta \gamma^2)+2$}
\State Let $\bw^{(0)}=\mathbf{0}$.
\For{$i=0, 1, \ldots, s-1$}
\State Let $\D^{(i)}$ be $\wh{\D}_m$ conditioned on  $y \langle \bw^{(i)},  \bx \rangle \leq \gamma/2$.
\State Let $\bM^{(i)} = \E_{(\bx,y)\sim \D^{(i)}} [\bx\bx^T]$.
\State Use SVD on $\bM^{(i)}$ to find a basis for $V_{k^{(i)}}$, the span of the top $k^{(i)}$ eigenvectors.
\State Let $C^{(i)}$ be a $\delta \gamma^3$-cover, in $\ell_2$-norm, of $V_{k^{(i)}}\cap \B_d$ of size
$(1/(\delta \gamma))^{O(k^{(i)})}$.
\State For each $\bp^{(i)} \in C^{(i)}$ repeat the next step of the for loop with $\bw^{(i+1)}=\bw^{(i)}+\bp^{(i)}$.
\EndFor
\EndFor
\State Let $C$ denote the set of all $\bw^{(i)}$ generated in the above loop.
\State Let $\bv \in \argmin_{\bw \in C} \err_{\gamma/2}^{\wh{D}_m}(\bw)$.
\State \Return $h_{\bv}(\bx) = \sgn(\langle \bv, \bx \rangle)$.
\end{algorithmic}
\end{algorithm}

To show the correctness of the algorithm, we begin by arguing that the set $C$ of candidate weight vectors 
produced has size $2^{\tilde O(1/(\delta \gamma^2))}$.
This is because there are only $2^{O(1/(\delta \gamma^2))}$ many possibilities for the sequence of $k^{(i)}$,
and for each such sequence the product of the sizes of the $C^{(i)}$ is
$(1/(\delta \gamma))^{O(\sum k^{(i)})} = 2^{\tilde O(1/(\delta \gamma^2))}.$
We note that, by the aforementioned analysis, for any choice of $k^{(0)},\ldots,k^{(i-1)}$
and $\bw^{(i)}$, we either have that $\err_{\gamma/2}^{\wh{\D}_m} (\bw^{(i)}) \leq (1+\delta) \opt_{\gamma}^{\wh{\D}_m}$
or there is a choice of $k^{(i)}$ and $\bp^{(i)} \in C^{(i)}$ such that
$$\|\bw^{\ast}-\bw^{(i)}-\bp^{(i)}\|_2^2 \leq \|\bw^{\ast}-\bw^{(i)} \|_2^2 - \delta k^{(i)} \gamma^2/8 +
O(\delta^2 \gamma^6) \;,$$
where we used \eqref{eqn:norm-decreases} and the fact that $C^{(i)}$ is a $\delta \gamma^3$-cover of $V_{k^{(i)}}$.
Following the execution path of the algorithm, we either find some $\bw^{(i)}$
with $\err_{\gamma/2}^{\wh{\D}_m} (\bw^{(i)}) \leq (1+\delta) \opt_{\gamma}^{\wh{\D}_m}$,
or we find a $\bw^{(i)}$ with
$$\|\bw^{\ast}-\bw^{(i)} \|_2^2 \leq 1-\left(\littlesum_{j=0}^{i-1} k^{(j)}\right) \delta \gamma^2/8 + O(\delta \gamma^4) \;,$$
where the last term is an upper bound for $\left(\littlesum_{j=0}^{i-1} k^{(j)}\right) \cdot O(\delta^2 \gamma^6)$.
Note that this sequence terminates in at most $O(1/(\delta \gamma^2))$ stages,
when it becomes impossible that $\sum k^{(j)} > 8/(\delta \gamma^2)+1$.
Thus, the output of our algorithm must contain some weight vector $\bv$ with
$\err_{\gamma/2}^{\wh{\D}_m} (\bv) \leq (1+\delta) \opt_{\gamma}^{\wh{\D}_m}$.
The proof now follows by an application of Fact~\ref{fact:erm-margin}.
\new{This completes the proof of Theorem~\ref{thm:constant-factor-alg}.}

\subsection{$\alpha$-Agnostic Proper Learning Algorithm} \label{ssec:alg-bicrit}

In this section, we show that if one wishes to obtain an $\alpha$-agnostic proper learner
 for some large $\alpha \gg 1$, one can obtain runtime exponential
 in $1/(\alpha\gamma)^2$ rather than $1/\gamma^2$.
Formally, we prove:

\begin{theorem}\label{alphaTheorem}
There is an algorithm that uses $\tilde O(1/(\eps^2\gamma^2))$ samples, runs in time
$\poly(d) \cdot (1/\eps)^{\tilde{O}(1/(\alpha \gamma)^2)}$ and is an
$\alpha$-agnostic proper learner for $\gamma$-margin halfspaces with probability $9/10$.
\end{theorem}

Let $\D$ be a distribution over $\B_d\times \{-1,1\}$.
Suppose that there exists a unit vector $\bw^{\ast} \in \R^d$ such that
$\pr_{(\bx,y) \sim \D}[y \langle \bw^{\ast}, \bx \rangle \geq \gamma] \geq 1- \opt_{\gamma}^{\D}$ for some $\opt_{\gamma}^{\D}>0$.
Suppose additionally that $\gamma,\eps>0$ and $\alpha>1$.
We will describe an algorithm that given sample access
to $\D$ along with $\gamma,\alpha,\eps$ and $\opt_{\gamma}^{\D}$,
draws $O(\log(\alpha/\eps)/(\gamma\eps)^2)$ samples,
runs in time $\poly(d) \cdot (1/\gamma\eps)^{\tilde O(1/(\alpha\gamma)^2)}$
and with probability at least $9/10$
returns a $\bw$ with $$\pr_{(\bx,y)\sim\D}[\sgn(\langle \bw, \bx \rangle) \neq y] < O(\alpha \cdot \opt_{\gamma}^{\D} +\eps) \;.$$

We begin by giving an algorithm that works if the distribution $\D$ is known explicitly.
We will be able to reduce to this case by using the empirical distribution over a sufficiently
large set of samples. That is, we start by establishing the following:

\begin{proposition}\label{alphaAlgProp}
Let $\D$ be an explicit distribution over $\B_d\times \{-1,1\}$.
Suppose there exists a unit vector $\bw^{\ast}$ so that
$\pr_{(\bx,y)\sim \D}[y \langle \bw^{\ast}, \bx \rangle \geq \gamma] \geq 1- \opt_{\gamma}^{\D}$
for some $\opt_{\gamma}^{\D}>0$. Additionally, let $\gamma>0$ and $\alpha>1$.
There exists an algorithm that given $\D$ along with $\gamma,\alpha,\opt_{\gamma}^{\D}$,
runs in time $\poly(d) \cdot (|\supp(\D)|/(\alpha\gamma \cdot \opt_{\gamma}^{\D}))^{\tilde O(1/(\alpha\gamma)^2)}$
and returns a weight vector $\bw$ with $\pr_{(\bx,y)\sim\D}[\sgn(\langle \bw, \bx \rangle )\neq y] < O(\alpha \cdot \opt_{\gamma}^{\D}).$
\end{proposition}


Our main technical tool here will be the vector of Chow parameters~\cite{Chow:61, OS11:chow, DeDFS14}, 
i.e., vector of degree-$1$ ``Fourier coefficients'', of the target halfspace:

\begin{definition} \label{def:chow}
Given a Boolean function $f: \B_d \to \{ \pm 1\}$, and a distribution $\D_{\bx}$ on $\B_d$
the {\em Chow parameters vector} of $f$, is the vector $\chow(f)$
given by the expectation $\E_{ \bx \sim \D_{\bx}}[f(\bx) \bx]$.
\end{definition}

\new{It is well-known~\cite{Chow:61} that the vector of Chow parameters uniquely identifies any halfspace
within the class of all Boolean functions. Several robust versions of this fact are known 
(see, e.g.,~\cite{Goldberg:06b, OS11:chow, DiakonikolasServedio:09, DeDFS14, DKS18a, DK19-degd}) 
under various structural assumptions on the underlying distribution.
Here we leverage the margin assumption to obtain a robust
version of this fact. Specifically,} we show that learning the Chow parameters of the halfspace
$f_{\bw^\ast}(\bx)=\sgn(\langle \bw^\ast, \bx \rangle)$ determines the function $f_{\bw^\ast}$ up to small error.

\new{In the following, we will denote by $\D_{\bx}$ the marginal distribution of $\D$ on $\B_d$.}
We have the following simple lemma:
\begin{lemma} \label{lem:chow-vs-dist}
Let $g: \B_d \to \{\pm 1\}$ be any Boolean function that satisfies 
$\pr_{\bx \sim \D_{\bx}}[f_{\bw^\ast}(\bx) \neq g(\bx)]\geq \nu + \opt_{\gamma}^{\D}$, for some $\nu > 0$.
Then, we have that $\|\chow(f_{\bw^\ast})-\chow(g)\|_2 \geq \nu \cdot \gamma$.
\end{lemma}
\begin{proof}
We can write
\begin{align*}
\|\chow(f_{\bw^\ast})-\chow(g)\|_2 &\geq \langle \bw^\ast, \chow(f_{\bw^\ast})-\chow(g) \rangle\\
& = \E_{\bx\sim \D_{\bx}}[\langle \bw^\ast, \bx \rangle (f_{\bw^\ast}(\bx)-g(\bx))]\\
& = 2\E_{\bx\sim \D_{\bx}}[|\langle \bw^\ast, \bx \rangle | \cdot \textbf{1}_{f_{\bw^\ast}(\bx)\neq g(\bx)}] \;.
\end{align*}
Recalling our assumptions
$\pr_{(\bx,y)\sim \D}[y \langle \bw^{\ast}, \bx \rangle \geq \gamma] \geq 1- \opt_{\gamma}^{\D}$
and $\pr_{\bx \sim \D_{\bx}}[f_{\bw^\ast}(\bx) \neq g(\bx)]\geq \nu + \opt_{\gamma}^{\D}$, we note that
there is at least a \new{$\nu$} probability \new{over $(\bx, y) \sim \D$} that $f_{\bw^\ast}(\bx)\neq g(\bx)$ and
$y \langle \bw^\ast, \bx \rangle \geq \new{\gamma}$, which implies that
$|\langle \bw^\ast, \bx \rangle |\geq$ \new{$\gamma$}. Therefore, the above expectation is at least $\nu \cdot \gamma$.
\end{proof}

Lemma~\ref{lem:chow-vs-dist}, combined with the algorithms in
~\cite{TTV:09short, DeDFS14}, implies that learning an approximation to $\chow(f_{\bw^\ast})$
is sufficient to learn a good hypothesis.
\begin{lemma} \label{ChowToWLem}
There is a polynomial time algorithm that given an explicit distribution $\D$
and a vector $\bc$ with $\|\chow(f_{\bw^\ast}) - \bc\|_2 \leq \nu \cdot \gamma$,
returns a vector $\bw$ that with high probability satisfies
$\pr_{(\bx, y)\sim \D}[f_{\bw}(\bx) \neq f_{\bw^\ast}(\bx)] \leq O(\nu+\opt_{\gamma}^{\D})$.
In particular, for this $\bw$ we have that
$\pr_{(\bx,y)\sim \D}[\sgn(\langle \bw, \bx \rangle) \neq y] = O(\nu +\opt_{\gamma}^{\D}).$
\end{lemma}

Thus, it will suffice to approximate the Chow parameters of $f_{\bw^{\ast}}$ to error
\new{$\alpha \gamma \cdot \opt_{\gamma}^{\D}$}. One might consider using the empirical Chow parameters,
namely $P=\E_{(\bx,y)\sim \D}[y\bx]$ for this purpose. In the realizable case,
this would be the right thing to do, but this naive approach fails in the agnostic setting.
Instead, our approach hinges on the following observation:
Since $y=f_{\bw^\ast}(\bx)$
for all but an $\opt_{\gamma}^{\D}$-fraction of $\bx$'s, and since the $\bx$'s are supported in the unit ball,
the error has $\ell_2$-norm at most $\opt_{\gamma}^{\D}$. In fact, if we have some vector
$\bw$ so that $\langle \bw, P-\chow(f_{\bw^{\ast}}) \rangle \geq \alpha \gamma \cdot \opt_{\gamma}^{\D} $,
then there must be some $(\bx,y)$ in the domain of $\D$ with
$|\langle \bx, \bw \rangle| \geq \alpha\gamma$. The idea is to guess this
$\bw$ and then guess the true projection of $\chow(f_{\bw^{\ast}})$ onto $\bw$.

We present the pseudo-code for the algorithm establishing
Proposition~\ref{alphaAlgProp} as Algorithm~\ref{alg:alpha-finite-support} below.
\begin{algorithm}
\caption{\label{alg:alpha-finite-support} $\alpha$-Agnostic Proper Learner of Proposition \ref{alphaAlgProp}}
\begin{algorithmic}[1]
\State Let $m = \lceil \log(1/\alpha\gamma)/(\alpha\gamma)^2 \rceil$.
\State Let $P=\E_{(\bx,y)\sim \D}[y\bx]$
\For{every sequence $\bx^{(1)},\ldots,\bx^{(m)}$ from $\supp(\D)$}
\State Let $V$ be the span of $\bx^{(1)},\ldots,\bx^{(m)}$.
\State Let $\mathcal{C}$ be a $(\alpha\gamma\cdot \opt_{\gamma}^{\D})$-cover of the unit ball of $V$.
\For{each $g\in\mathcal{C}$}
\State Let $P'$ be obtained by replacing the projection of $P$ onto $V$ with $g$. In particular, $P' = P-\mathrm{Proj}_V(P)+g$.
\State Run the algorithm of Lemma \ref{ChowToWLem} to find a hypothesis $\bw$. \label{line:best-chow}
\EndFor
\EndFor
\State \Return The hypothesis that produces smallest empirical error among all $\bw$'s in Line~\ref{line:best-chow}.
\end{algorithmic}
\end{algorithm}

\begin{proof}[Proof of Proposition \ref{alphaAlgProp}]
Firstly, note that the runtime of this algorithm is clearly
$\poly(d)\left(\frac{|\supp(\D)|}{\opt_{\gamma}^{\D} \cdot \alpha\gamma} \right)^{\tilde O(1/(\alpha\gamma)^2)}$.
It remains to show correctness. We note that by Lemma~\ref{ChowToWLem}
it suffices to show that some $P'$ is within $O(\alpha  \gamma\cdot \opt_{\gamma}^{\D})$ of $\chow(f_{\bw^{\ast}})$.
For this it suffices to show that there is a sequence $\bx^{(1)},\ldots,\bx^{(m)}$
so that $\|\proj_{V^\perp}(\chow(f_{\bw^{\ast}})-P)\|_2 = O(\alpha \gamma \cdot \opt_{\gamma}^{\D})$.

To show this, let $V_i$ be the span of $\bx^{(1)}, \bx^{(2)},\ldots,\bx^{(i)}$.
We claim that if $\|\proj_{V_i^\perp}(\chow(f_{\bw^{\ast}})-P)\|_2 \geq \alpha \gamma \cdot \opt_{\gamma}^{\D}$,
then there exists an $\bx^{(i+1)}$ in the support of $\D$ such that
$$
\|\proj_{V_{i+1}^\perp}(\chow(f_{\bw^{\ast}})-P)\|_2^2 = \|\proj_{V_i^\perp}(\chow(f_{\bw^{\ast}})-P)\|_2^2 \cdot (1 - (\alpha\gamma)^2).
$$
To show this, we let $\bw$ be the unit vector in the direction of
$\proj_{V_i^\perp}(\chow(f_{\bw^{\ast}})-P)$. We note that
$$\|\proj_{V_i^\perp}(\chow(f_{\bw^{\ast}})-P)\|_2 = \langle \bw, \chow(f_{\bw^{\ast}})-P \rangle
= \E_{(\bx,y)\sim \D}[\langle \bw, \bx \rangle  (\sgn(\langle \bw^{\ast}, \bx \rangle)-y)] \;.$$
Since $\sgn(\langle \bw^{\ast}, \bx \rangle)-y$ is $0$ for all but an $\opt_{\gamma}^{\D}$-fraction of $(\bx,y)$,
we have that there must be some $\bx^{(i+1)}$
so that $\langle \bx^{(i+1)}, \bw \rangle \geq \|\proj_{V_i^\perp}(\chow(f_{\bw^{\ast}})-P)\|_2/\opt_{\gamma}^{\D} \geq \alpha \gamma$.
If we chose  this $\bx^{(i+1)}$, we have that
\begin{align*}
\|\proj_{V_{i+1}^\perp}(\chow(f_{\bw^{\ast}})-P)\|_2^2
& \leq \|\proj_{V_{i}^\perp}(\chow(f_{\bw^{\ast}})-P)\|_2^2 - \langle \bx^{(i+1)}, \chow(f_{\bw^{\ast}})-P \rangle ^2\\
& = \|\proj_{V_{i}^\perp}(\chow(f_{\bw^{\ast}})-P)\|_2^2 \cdot (1-\langle \bx^{(i+1)}, \bw \rangle^2)\\
& = \|\proj_{V_i^\perp}(\chow(f_{\bw^{\ast}})-P)\|_2^2 \cdot (1 - (\alpha\gamma)^2).
\end{align*}
Therefore, unless $\|\proj_{V_i^\perp}(\chow(f_{\bw^{\ast}})-P)\|_2^2 < \alpha\gamma \cdot \opt_{\gamma}^{\D}$ already for some $i < m$, there exists a sequence $\bx^{(1)}, \bx^{(2)}, \cdots, \bx^{(m)}$ such that
\begin{align*}
\|\proj_{V_m^\perp}(\chow(f_{\bw^{\ast}})-P)\|_2^2
&\leq \|P-\chow(f_{\bw^{\ast}})\|_2^2 \cdot (1 - (\alpha \gamma)^2)^{-m} \\
&\leq \|P-\chow(f_{\bw^{\ast}})\|_2^2 \cdot \exp(-m \cdot (\alpha\gamma)^2)\\
&\leq \opt_{\gamma}^{\D} \cdot \exp(\log(\alpha \gamma)) \\
&= \opt_{\gamma}^{\D} \cdot \alpha\gamma.
\end{align*}
So in either case, we have some sequence of $\bx$'s so that the projection
onto $V^\perp$ of $\chow(f_{\bw^{\ast}})-P$ is sufficiently small. This completes our analysis.
\end{proof}

In order to extend this to a proof of Theorem \ref{alphaTheorem},
we will need to reduce to solving the problem on a finite sample set.
This result can be obtained from Proposition \ref{alphaAlgProp} by some fairly simple reductions.

Firstly, we note that we can assume that $\opt_{\gamma}^{\D} \geq \eps/\alpha$,
as increasing it to this value does not change the problem.

Secondly, we note that if we let $\wh{\D}$ be the empirical distribution over a set of $\Omega(d/\eps^2)$
random samples, then with at least $2/3$ probability
we have the following:
\begin{itemize}
\item $\pr_{(\bx,y)\sim \wh{\D}}[y \langle \bw^{\ast}, \bx \rangle  \geq \gamma] \geq 1- O(\opt_{\gamma}^{\D})$.
\item For any vector $\bw$, $\pr_{(\bx,y)\sim\D}[\sgn( \langle \bw, \bx \rangle )\neq y]
= \pr_{(\bx,y)\sim \wh{\D}}[\sgn( \langle \bw, \bx \rangle)\neq y] + O(\eps)$.
\end{itemize}
The first statement here is by applying the Markov inequality to the probability that
$y \langle \bw^{\ast}, \bx \rangle < \gamma$, and the second is by the VC-inequality~\cite{DL:01}.
We note that if the above hold, applying the algorithm from Proposition \ref{alphaAlgProp}
to $\wh{\D}$ will produce an appropriate $\bw$. This produces an algorithm that uses $O(d/\eps^2)$
samples and has runtime $O(d /\gamma\eps)^{\tilde O(1/(\alpha\gamma)^2)}.$

Unfortunately, this algorithm is not quite satisfactory as the runtime and sample complexity scale poorly with the dimension $d$.
In order to fix this, we will make use of an idea from~\cite{KlivansServedio:04coltmargin}.
Namely, we will first apply dimension reduction to a smaller number of dimensions before applying our algorithm.
In particular, we will make use of the Johnson-Lindenstrauss lemma:
\begin{lemma}[\cite{JohnsonLindenstrauss:84}]
There exists a probability distribution over linear transformations
$A:\R^d\rightarrow \R^m$ with $m=O(\log(1/\delta)/\eps^2)$ so that for any unit vectors $\bv, \bw\in \R^d$,
$\pr_A[|\langle \bv, \bw \rangle - \langle A \bv, A\bw \rangle | > \eps] < \delta.$
Additionally, there are efficient algorithms to sample from such distributions over $A$.
\end{lemma}
We note that this implies in particular that $\|A\bv\|_2 = 1\pm \eps$ except for with probability $\delta$.
Thus, by tweaking the parameters a little bit and letting $h_A(\bv) = A\bv/\|A\bv\|_2$,
we have that $h_A(\bv)$ is always a unit vector and that $\langle h_A(\bv), h_A(\bw) \rangle= \langle \bv, \bw\rangle \pm \eps$
except with probability $\delta.$

Next, we note that by taking $\eps=\gamma/2$ and $\delta = \opt_{\gamma}^{\D}$ in the above we have that
\begin{eqnarray*}
&&\pr_{A,(\bx,y)\sim \D}[y \langle h_A(\bw^{\ast}), h_A(\bx) \rangle < \gamma/2] \\
&\leq& \pr_{(\bx,y)\sim \D}[y\langle \bw^{\ast}, \bx \rangle <\gamma] +
\pr_{A,(\bx,y)\sim \D}[| \langle h_A(\bw^{\ast}), h_A(\bx) \rangle - \langle \bw^{\ast}, \bx \rangle|>\gamma/2] \\
&=& O(\opt_{\gamma}^{\D}).
\end{eqnarray*}
Thus, by the Markov inequality, with large constant probability over $A$ we have that
$$
\pr_{(\bx,y) \sim \D}[y \langle h_A(\bw^{\ast}), h_A(\bx)\rangle < \gamma/2] = O(\opt_{\gamma}^{\D}).
$$
But this means that the distribution $(h_A(\bx),y)$ satisfies the assumptions
for our algorithm (with $\gamma$ replaced by $\gamma/2$ and $\opt_{\gamma}^{\D}$ by $O(\opt_{\gamma}^{\D})$),
but in dimension $m=O(\log(\alpha/\eps)/\gamma^2)$. Running the algorithm described above on this set
will find us a vector $\bw$ so that
$$
\pr_{(\bx,y)\sim \D}[\sgn(\langle \bw, h_A(\bx) \rangle)\neq y] = O(\alpha \cdot \opt_{\gamma}^{\D}+\eps).
$$
However, it should be noted that
$$\sgn(\langle \bw, h_A(\bx) \rangle) = \sgn(\langle \bw, A\bx \rangle /\|A\bx|_2)
= \sgn(\langle \bw, A\bx \rangle) = \sgn(\langle A^T  \bw, \bx\rangle) \;.$$
Thus, $A^T\bw$ satisfies the necessary conditions.

Our final algorithm is given below:
\begin{algorithm}
\caption{$\alpha$-Agnostic Proper Learner of Theorem~\ref{alphaTheorem}}
\begin{algorithmic}[1]
\State Pick $A:\R^d\rightarrow \R^m$ with $m=O(\log(\alpha/\eps)/\gamma^2)$
from an appropriate Johnson-Lindenstrauss family and define $f_A$ appropriately.
\State Take $O(m/\eps^2)$ random samples and let $\wh{\D}$ be the uniform distribution over
$(A\bx/\|A\bx|_2,y)$ for samples $(\bx,y)$ from this set.
\State Run the algorithm from Proposition \ref{alphaAlgProp} on $\wh{\D}$ using $\gamma/2$ instead of $\gamma$ to find a vector $\bw$.
\State \Return $A^T\bw$.
\end{algorithmic}
\end{algorithm}

\section{Computational Hardness Results} \label{sec:lb}

In this section, we provide several computational lower bounds for agnostic learning of halfspaces with a margin. 
To clarify the statements below, we note that we say ``there is no algorithm that runs in time $T(d, \frac{1}{\gamma}, \frac{1}{\varepsilon})$'' to mean that no $T(d, \frac{1}{\gamma}, \frac{1}{\varepsilon})$-time algorithm works for \emph{all} combinations of parameters $d,\gamma$ and $\varepsilon$. (Note that we discuss the lower bounds with stronger quantifiers in Section~\ref{sec:lb-strong-quantifier}.) Moreover, we also ignore the dependency on $\tau$ (the probability that the learner can be incorrect), since we only use a fixed $\tau$ (say $1/3$) in all the bounds below.

First, we show that, for any constant $\alpha > 1$, $\alpha$-agnostic learning of $\gamma$-margin halfspaces 
requires $2^{(1/\gamma)^{2- o(1)}}\poly(d,1/\varepsilon)$ time. 
Up to the lower order term $\gamma^{o(1)}$ in the exponent, this matches the runtime of our algorithm 
(in Theorem~\ref{thm:constant-factor-alg}). In fact, we show an even stronger result, namely that if the dependency 
of the running time on the margin is, say, $2^{(1/\gamma)^{1.99}}$, then one has to pay a nearly exponential dependence on $d$,
i.e., $2^{d^{1 - o(1)}}$. 

This result holds assuming the so-called (randomized) exponential time hypothesis (ETH)~\cite{ImpagliazzoP01,ImpagliazzoPZ01}, which postulates that there is no (randomized) algorithm that can solve 3SAT in time $2^{o(n)}$, where $n$ denotes the number of variables. ETH is a standard hypothesis used in proving (tight) running time lower bounds. We do not discuss ETH further here, but interested readers may refer to a survey by Lokshtanov et al.~\cite{LokshtanovMS11} for an in-depth discussion and several applications of ETH.

Our first lower bound can be stated more precisely as follows:

\begin{theorem} \label{thm:run-time}
Assuming the (randomized) ETH, for any universal constant $\alpha \geq 1$, there is no proper $\alpha$-agnostic learner for $\gamma$-margin halfspaces that runs in time $O(2^{(1/\gamma)^{2-o(1)}}2^{d^{1 - o(1)}})f(\frac{1}{\varepsilon})$ for any function $f$. 
\end{theorem}

Secondly, we address the question of whether we can achieve $\alpha = 1$ (standard agnostic learning) 
while retaining running time similar to that of our algorithm. We answer this in the negative 
(assuming a standard parameterized complexity assumption): there is no $f(\frac{1}{\gamma}) \poly(d,\frac{1}{\varepsilon})$-time 
$1$-agnostic learner for any function $f$ (e.g., even for $f(\frac{1}{\gamma}) = 2^{2^{2^{1/\gamma}}}$). 
This demonstrates a stark contrast between what we can achieve with and without approximation.

\begin{theorem} \label{thm:param}
Assuming W[1] is not contained in randomized FPT, there is no proper $1$-agnostic learner for $\gamma$-margin halfspaces that runs in time $f(\frac{1}{\gamma})\poly(d,\frac{1}{\varepsilon})$ for any function $f$.
\end{theorem}

Finally, we explore the other extreme of the trade-off between the running time and approximation ratio, by asking: what is the best approximation ratio we can achieve if we only consider proper learners that run in $\poly(d,\frac{1}{\varepsilon},\frac{1}{\gamma})$-time? On this front, it is known~\cite{Servedio:01lnh} that the perceptron algorithm achieves $1/\gamma$-approximation. We show that a significant improvement over this is unlikely, by showing that $(1/\gamma)^{\frac{1}{\polyloglog(1/\gamma)}}$-approximation is not possible unless NP = RP. If we additionally assume the so-called Sliding Scale Conjecture~\cite{BGLR94}, this ratio can be improved to $(1/\gamma)^{c}$ for some constant $c > 0$.

\begin{theorem} \label{thm:inapx}
Assuming NP $\ne$ RP, there is no proper $(1/\gamma)^{1/\polyloglog(1/\gamma)}$-agnostic learner for $\gamma$-margin halfspaces that runs in time $\poly(d,\frac{1}{\varepsilon},\frac{1}{\gamma})$. Furthermore, assuming NP $\ne$ RP and the Sliding Scale Conjecture (Conjecture~\ref{conj:ssc}), there is no proper $(1/\gamma)^c$-agnostic learning for $\gamma$-margin halfspaces that runs in time $\poly(d,\frac{1}{\varepsilon},\frac{1}{\gamma})$ for some constant $c > 0$.
\end{theorem}

We note here that the constant $c$ in Theorem~\ref{thm:inapx} is not explicit, i.e., it depends on the constant from the Sliding Scale Conjecture (SSC). Moreover, even when assuming the most optimistic parameters of SSC, the constant $c$ we can get is still very small. For instance, it is still possible that a say $\sqrt{1/\gamma}$-agnostic learning algorithm that runs in polynomial time exists, and this remains an interesting open question. We remark that Daniely et al.~\cite{DanielyLS14b} have made partial progress in this direction 
by showing that, any $\poly(d,\frac{1}{\varepsilon},\frac{1}{\gamma})$-time learner that belongs to a ``generalized linear family'' cannot achieve approximation ratio $\alpha$ better than $\Omega\left(\frac{1/\gamma}{\polylog(1/\gamma)}\right)$. We note that the inapproximability ratio of \cite{DanielyLS14b} is close to being tight for a natural, yet restricted, family of improper learners. 
On the other hand, our proper hardness result holds against {\em all} proper learners under a widely believed worst-case
complexity assumption.

\subsection{Lower Bounds with Stronger Quantifiers on Parameters}
\label{sec:lb-strong-quantifier}

Before we proceed to our proofs, let us first state a running time lower bound with stronger quantifiers. 
Recall that previously we only rule out algorithms that work \emph{for all} combinations of $d, \gamma, \varepsilon$. 
Below we relax the quantifier so that we need the \emph{for all} quantifier only for $d$.

\begin{lemma} \label{lem:strong-quantifier}
Assuming the (randomized) ETH, for any universal constant $\alpha \geq 1$, there exists $\varepsilon_0 = \varepsilon_0(\alpha)$ such that there is no $\alpha$-agnostic learner for $\gamma$-margin halfspaces that runs in time $O(2^{(1/\gamma)^{2 - o(1)}})\poly(d)$ for all $d$ and for some $0 < \varepsilon < \varepsilon_0$ and $\frac{1}{d^{0.5 - o(1)}} \leq \gamma = \gamma(d) \leq \frac{1}{(\log d)^{0.5 + o(1)}}$ that satisfies $\frac{\gamma(d + 1)}{\gamma(d)} \geq \Omega(1)$.
\end{lemma}

We remark here that the lower and upper bounds on $\gamma$ are essentially (i.e., up to lower order terms) the best possible. On the upper bound front, if $\gamma \geq \tilde{O}\left(\frac{1}{\sqrt{\log d}}\right)$, then our algorithmic result (Theorem~\ref{thm:constant-factor-alg}) already give a $\poly(d, \frac{1}{\varepsilon})$-time $\alpha$-agnostic learner for $\gamma$-margin halfspaces (for all constant $\alpha > 1$). On the other hand, if $\gamma \leq O(\frac{1}{d^{0.5 + o(1)}})$, then the trivial algorithm that exactly solves ERM for $m = O\left(\frac{d}{\varepsilon^2}\right)$ samples only takes $2^{O(d/\varepsilon^2)}$ time, which is already asymptotically faster than $2^{(1/\gamma)^{2 - o(1)}}$. The last condition that $\frac{\gamma(d + 1)}{\gamma(d)}$ is not too small is a sanity-check condition that prevents ``sudden jumps'' in $\gamma(d)$ such as $\gamma(d) = \frac{1}{(\log d)^{0.1}}$ and $\gamma(d + 1) = \frac{1}{(d + 1)^{0.1}}$; note that the condition is satisfied by ``typical functions'' such as $\gamma(d) = \frac{1}{d^c}$ or $\gamma(d) = \frac{1}{(\log d)^c}$ for some constant $c$.

As for $\varepsilon$, we only require the algorithm to work for any $\varepsilon$ that is not \emph{too large}, i.e., no larger than $\varepsilon_0(\alpha)$. This latter number is just a constant (when $\alpha$ is a constant). We note that it is still an interesting open question to make this requirement as mild as possible; specifically, is it possible to only require the algorithm to work for any $\varepsilon < 1/2$?

\subsection{Reduction from $k$-Clique and Proof of Theorem~\ref{thm:param}}

We now proceed to the proofs of our results, starting with Theorem~\ref{thm:param}.

To prove Theorem~\ref{thm:param}, we reduce from the $k$-Clique problem. In $k$-Clique, we are given a graph $G$ and an integer $k$, and the goal is to determine whether the graph $G$ contains a $k$-clique (as a subgraph). 

We take the perspective of \emph{parameterized complexity}. Recall that a parameterized problem with parameter $k$ is said to be \emph{fixed parameter tractable (FPT)} if it can be solved in time $f(k)\poly(n)$ for some computable function $f$, where $n$ denotes the input size. 

It is well-known that $k$-Clique is complete for the class W[1]~\cite{DowneyF95}. In other words, under the (widely-believed) assumption that W[1] does not collapse to FPT (the class of fixed parameter tractable problems), we cannot solve $k$-Clique in time $f(k) \poly(n)$ for any computable function $f$. We shall not formally define the class W[1] here; interested readers may refer to the book of Downey and Fellows for 
an in-depth treatment of the topic~\cite{DowneyF13}.

Our reduction starts with an instance of $k$-Clique and produces an instance of agnostic learning with margin $\gamma$ such that $\gamma = \Omega(1/k)$ (and the dimension is polynomial):

\begin{lemma} \label{lem:clique-red}
There exists a polynomial-time algorithm that takes as input an $n$-vertex graph instance $G$ and an integer $k$, and produces a distribution $\D$ over $\B_d \times \{\pm 1\}$ and $\gamma, \kappa \in [0, 1]$ such that
\begin{itemize}
\item (Completeness) If $G$ contains a $k$-clique, then $\opt_{\gamma}^{\D} \leq \kappa$.
\item (Soundness) If $G$ does not contains a $k$-clique, then $\opt_{0-1}^{\D} > \kappa + \frac{0.001}{n^3}$.
\item (Margin Parameter) $\gamma \geq \Omega(\frac{1}{\sqrt{k}})$.
\end{itemize}
\end{lemma}

We remark here that, in Lemma~\ref{lem:clique-red} and throughout the remainder of the section, we say that an algorithm produces a distribution $\D$ over $\B_d \times \{\pm 1\}$ to mean that it outputs the set of samples $\{(\bx^{(i)}, y^{(i)})\}_{i \in [m]}$ and numbers $d_i$ for each $i \in [m]$ representing the probability of $(\bx^{(i)}, y^{(i)})$ with respect to $\D$. Note that this is stronger than needed since, to prove hardness of learning, it suffices to have an oracle that can sample from $\D$, but here we actually explicitly produce a full description of $\D$. Moreover, note that this implicitly implies that the support of $\D$ is of polynomial size (and hence, for any given $h$, $\err_{\gamma}^{\D}(h)$ and $\err_{0-1}^{\D}(h)$ can be efficiently computed).

As stated above, Lemma~\ref{lem:clique-red} immediately implies Theorem~\ref{thm:param} because, if we can agnostically learn $\gamma$-margin halfspaces in time $f(\frac{1}{\gamma})\poly(d,\frac{1}{\varepsilon})$, then we can solve $k$-Clique in $f(O(\sqrt{k})) \poly(n)$ time, which would imply that W[1] is contained in (randomized) FPT. This is formalized below.

\begin{proof}[Proof of Theorem~\ref{thm:param}]
Suppose that we have an $f(\frac{1}{\gamma})\poly(d,\frac{1}{\varepsilon})$-time agnostic learner for $\gamma$-margin halfspaces. Given an instance $(G, k)$ of $k$-Clique, we run the reduction from Lemma~\ref{lem:clique-red} to produce a distribution $\D$. We then run the learner on $\D$ with $\varepsilon = \frac{0.001}{n^3}$ (and with $\delta = 1/3$). Note that the learner runs in time $f(O(\sqrt{k}))\poly(n)$ and produces a halfspace $h$. We then compute $\err_{0-1}^{\D}(h)$; if it is no more than $\kappa + \frac{0.001}{n^3}$, then we output YES. Otherwise, we 
output NO.

The algorithm described above solves $k$-Clique (correctly with probability 2/3) in FPT time. Since $k$-Clique is W[1]-complete, this implies that W[1] is contained in randomized FPT.
\end{proof}

We now move on to prove Lemma~\ref{lem:clique-red}. Before we do so, let us briefly describe the ideas behind it. The dimension $d$ will be set to $n$, the number of vertices of $G$. Each coordinate $\bw_i$ is associated with a vertex $i \in V(G)$. In the completeness case, we would like to set $\bw_i = \frac{1}{\sqrt{k}}$ iff $i$ is in the $k$-clique and $\bw_i = 0$ otherwise. To enforce a solution to be of this form, we add two types of samples that induces the following constraints:
\begin{itemize}
\item \emph{Non-Edge Constraint:} for every \emph{non-}edge $(i, j)$, we should have $\bw_i + \bw_j \leq \frac{1}{\sqrt{k}}$. That is, we should ``select'' at most one vertex among $i, j$.
\item \emph{Vertex Selection Constraint:} each coordinate of $\bw$ is at least $\frac{1}{\sqrt{k}}$. Note that we will violate such constraints for all vertices, except those that are ``selected''.
\end{itemize}

If we select the probabilities in $\D$ so that the non-edge constraints are weighted much larger than the vertex selection constraints, then it is always better to not violate any of the first type of constraints. When this is the case, the goal will now be to violate as few vertex selection constraints as possible, which is the same as finding a maximum clique, as desired. 

While the above paragraph describes the core idea of the reduction, there are two additional issues we have to resolve:
\begin{itemize}
\item \emph{Constant Coordinate: } first, notice that we cannot actually quite write a constraint of the form $\bw_i + \bw_j \leq \frac{1}{\sqrt{k}}$ using the samples because there is no way to express a value like $\frac{1}{\sqrt{k}}$ directly. To overcome this, we have a ``constant coordinate'' $\bw_*$, which is supposed to be a constant, and replace the right hand side of non-edge constraints by $\frac{\bw_*}{\sqrt{k}}$ (instead of $\frac{1}{\sqrt{k}}$). The new constraint can now be represented by a sample.
\item \emph{Margin:} in the above reduction, there was no margin at all! To get the appropriate margin, we ``shift'' the constraint slightly so that there is a margin. For instance, instead of $\frac{\bw_*}{\sqrt{k}}$ for a non-edge constraint, we use $\frac{1.1 \bw_*}{\sqrt{k}}$.  We now have a margin of $\approx \frac{0.1}{\sqrt{k}}$ and it is still possible to argue that the best solution is still to select a clique.
\end{itemize}

The reduction, which follows the above outline, is formalized below.

\begin{proof}[Proof of Lemma~\ref{lem:clique-red}]
Given a graph $G = (V, E)$, we use $n$ to denote the number of vertices $|V|$ and we rename its vertices so that $V = [n]$. We set $d = n + 1$; we name the first coordinate $*$ and each of the remaining coordinates $i \in [n]$. For brevity, let us also define $\beta = 1 - \frac{0.01}{n^2}$. The distribution $\D$ is defined as follows:
\begin{itemize}
\item Add a labeled sample $(-\be^*, -1)$ with probability $\frac{\beta}{2}$ in $\D$. We refer to this as the \emph{positivity constraint for $*$}.
\item For every pair of distinct vertices $i, j$ that do not induce an edge in $E$, add a labeled sample $(\frac{1}{2}\left(\frac{1.1}{\sqrt{k}}\be^* - \be^{i} - \be^{j}\right), 1)$ with probability $\frac{\beta}{2\left(\binom{n}{2} - |E|\right)}$ in $\D$. We refer to this as the \emph{non-edge constraint for $(i, j)$}.
\item For every vertex $i$, add a labeled sample $(\frac{1}{2}\left(\be^{i} - \frac{0.9}{\sqrt{k}}\be^*\right), 1)$ with probability $\frac{0.01}{n^3}$ in $\D$. We refer to this as the \emph{vertex selection constraint for $i$}.
\end{itemize}
Finally, let $\gamma = \frac{0.1}{2\sqrt{2k}}$, $\kappa = (n - k)\left(\frac{0.01}{n^3}\right)$. It is obvious that the reduction runs in polynomial time.

\paragraph{Completeness.} Suppose that $G$ contains a $k$-clique; let $S \subseteq V$ denote the set of its vertices. We define $\bw$ by $\bw_* = \frac{1}{\sqrt{2}}$ and, for every $i \in V$, $\bw_i = \frac{1}{\sqrt{2k}}$ if $i \in S$ and $\bw_i = 0$ otherwise. It is clear that $\|\bw\|_2 = 1$ and that, for every $(\bx, y) \in \supp(\D)$, we have $|\left<\bw, \bx\right>| \geq \frac{0.1}{2\sqrt{2k}}$. Finally, observe that all the first two types of constraints are satisfied, and a vertex selection constraint for $i$ is unsatisfied iff $i \notin S$. Thus, we have $\err_\gamma^{\D}(\bw) = (n - k)\left(\frac{0.01}{n^3}\right) = \kappa$, which implies that $\opt_\gamma^{\D} \leq \kappa$ as desired.

\paragraph{Soundness.} Suppose contrapositively that $\opt_{0-1}^{\D} \leq \kappa + \frac{0.001}{n^3}$; that is, there exists $\bw$ such that $\err_{0-1}^{\D}(\bw) \leq \kappa + \frac{0.001}{n^3}$. Observe that each labeled sample of the first two types of constraints has probability more than $\frac{\beta}{2n^2} > \kappa + \frac{0.001}{n^3}$. As a result, $\bw$ must correctly classifies these samples. Since $\bw$ correctly classifies $(-\be^*, -1)$, it must be that $w_* > 0$.

Now, let $T$ be the set of vertices $i$ such that $\bw$ mislabels the vertex selection constraint for $i$. Observe that $|T| < \frac{\left(\kappa + \frac{0.001}{n^3}\right)}{\frac{0.01}{n^3}} < n - k + 1$. In other words, $S := V \setminus T$ is of size at least $k$. We claim that $S$ induces a $k$-clique in $G$. To see that this is true, consider a pair of distinct vertices $i, j \in S$. Since $\bw$ satisfies the vertex selection constraints for $i$ and for $j$, we must have $\bw_i, \bw_j \geq \frac{0.9}{\sqrt{k}}$. This implies that $(i, j)$ is an edge, as otherwise $\bw$ would mislabel the non-edge constraint for $(i, j)$.

As a result, $G$ contains a $k$-clique as desired.
\end{proof}

\subsection{Reduction from $k$-CSP and Proofs of Theorems~\ref{thm:run-time},~\ref{thm:inapx} and Lemma~\ref{lem:strong-quantifier}}

In this section, we will prove Theorems~\ref{thm:run-time} and~\ref{thm:inapx}, by reducing from the hardness of approximation of constraint satisfaction problems (CSPs), given by PCP Theorems. 

\subsubsection{CSPs and PCP Theorem(s)}

Before we can state our reductions, we have to formally define CSPs and state the PCP theorems we will use more formally. We start with the definition of $k$-CSP:
\begin{definition}[$k$-CSP]
For any integer $k \in \mathbb{N}$, a $k$-CSP instance $\cL = (V, \Sigma, \{\Pi_q\}_{q \in \cQ})$ consists of
\begin{itemize}
\item The variable set $V$,
\item The alphabet $\Sigma$, which we sometimes refer to as labels,
\item Constraints set $\{\Pi_S\}_{S \in \cQ}$, where $\cQ \subseteq \binom{V}{k}$ is a collection of $k$-size subset of $V$. For each subset $S = \{v_1, \dots, v_k\}$, $\Pi_S \subseteq \Sigma^S$ is the set of accepting answers for the constraint $\Pi_S$.

Here we think of each $f \in \Sigma^S$ as a function from $f: S \to \Sigma$.
\end{itemize}

A $k$-CSP instance is said to be \emph{regular} if each variable appears in the same number of constraints.

An assignment $\phi$ is a function $\phi: V \to \Sigma$. Its value, denoted by $\val_{\cL}(\phi)$, is the fraction of constraints $S \in \cQ$ such that\footnote{We use $\phi|_S$ to denote the restriction of $\phi$ on the domain $S$.} $\phi|_S \in \Pi_S$. Such constraints are said to be \emph{satisfied by $\phi$.} The value of $\cL$, denoted by $\val(\cL)$, is the maximum value among all assignments, i.e., $\val(\cL) := \max_{\phi} \val_{\cL}(\phi)$.

In the \textsc{$\nu$-Gap-$k$-CSP} problem, we are given a regular instance $\cL$ of $k$-CSP, and we want to distinguish between $\val(\cL) = 1$ and $\val(\cL) < \nu$.
\end{definition}

Throughout this subsection, we use $n$ to denote the instance size of $k$-CSP, that is $n = \sum_{S \in \cQ} |\Pi_S|$.

The celebrated PCP theorem~\cite{AroraS98,AroraLMSS98} is equivalent to the proof of NP-hardness of approximating $\nu$-Gap-$k$-CSP for some constant $k$ and $\nu < 1$. Since we would like to prove (tight) running time lower bounds, we need the versions of PCP Theorems that provides strong running time lower bounds as well. For this task, we turn to the Moshkovitz-Raz PCP theorem, which can not only achieve arbitrarily small constant $\nu > 0$ but also almost exponential running time lower bound.

\begin{theorem}[Moshkovitz-Raz PCP~\cite{MR10}] \label{thm:mr-pcp}
Assuming ETH, for any $0 < \nu < 1$, $\nu$-Gap-2-CSP cannot be solved in time $O(2^{n^{1 - o(1)}})$, even for instances with $|\Sigma| = O_\nu(1)$.
\end{theorem}

As for our hardness of approximation result (Theorem~\ref{thm:inapx}), we are aiming to get as large a ratio as possible. For this purpose, we will use a PCP Theorem of Dinur, Harsha and Kindler, which achieves $\nu = \frac{1}{\poly(n)}$ but need $k$ to be $\polyloglog(n)$.

\begin{theorem}[Dinur-Harsha-Kindler PCP~\cite{DHK15}] \label{thm:dhk-pcp}
$n^{-\Omega(1)}$-Gap-$\polyloglog(n)$-CSP is NP-hard.
\end{theorem}

Finally, we state the Sliding Scale Conjecture (SSC) of Bellare et al.~\cite{BGLR94}, which says that the NP-hardness with $\nu = \frac{1}{\poly(n)}$ holds even in the case where $k$ is constant: 

\begin{conj}[Sliding Scale Conjecture~\cite{BGLR94}] \label{conj:ssc}
For some constant $k$, $n^{-\Omega(1)}$-Gap-$k$-CSP is NP-hard.
\end{conj}


\subsubsection{Reducing from $k$-CSP to Agnostically Learning Halfspaces with Margin}

Having set up the notation, we now move on to the reduction from $k$-CSP to agnostic learning of halfspaces with margin. Our reduction can be viewed as a modification of the reduction from~\cite{ABSS97}; compared to~\cite{ABSS97}, we have to (1) be more careful so that we can get the margin in the completeness and (2) modify the reduction to work even for $k > 2$.

Before we precisely state the formal properties of the reduction, let us give a brief informal intuition behind the reduction. Given an instance $\cL = (V, \Sigma, \{\Pi_S\}_{S \in \cQ})$ of $k$-CSP, we will create a distribution $\D$ on $\B_d \times \{\pm 1\}$, where the dimension $d$ is equal to $n$. Each coordinate is associated with an accepting answer of each constraint; that is, each coordinate is $(S, f)$ where $S \in \cQ$ and $f \in \Pi_S$. In the completeness case where we have a perfect assignment $\phi$, we would like the halfspace's normal vector to set $\bw_{(S, f)} = 1$ iff $f$ is the assignment to predicate $S$ in $\phi$ (i.e., $f = \phi|_S$), and zero otherwise. To enforce this, we add three types of constraints:
\begin{itemize}
\item \emph{Non-negativity Constraint}: that each coordinate of $\bw$ should be non-negative.
\item \emph{Satisfiability Constraint}: that for each $S \in \cQ$, $\bw_{(S, f)}$ is positive for at least one $f \in \Pi_S$. 
\item \emph{Selection Constraint}: for each variable $v \in V$ and label $\sigma \in \Sigma$, we add a constraint that the sum of all $\bw_{(S, f)}$, for all $S$ that $v$ appears in and all $f$ that assigns $\sigma$ to $v$, is non-positive.
\end{itemize}
Notice that, for the completeness case, we satisfy the first two types of constraints, and we violate the selection constraints only when $\phi(v) = \sigma$. Intuitively, in the soundness case, we will not be able to ``align'' the positive $\bw_{(S, f)}$ from different $S$'s together, and we will have to violate a lot more selection constraints.

Of course, there are many subtle points that the above sketch does not address, such as the margin; on this front, we add one more special coordinate $\bw_*$, which we think of as being equal to 1, and we add/subtract $\delta$ times this coordinate to each of the constraints, which will create the margin for us. Another issue is that the normal vector of the halfspace (and samples) as above have norm more than one. Indeed, our assignment in the completeness case has norm $O(\sqrt{n})$. Hence, we have to scale the normal vector down by a factor of $O(\sqrt{n})$, which results in a margin of $\gamma = \Omega(1/\sqrt{n})$. This is the reason why we arrive at the running time lower bound of the form $2^{\gamma^{2 - o(1)}}$.

The properties and parameter dependencies of the reduction are encapsulated in the following theorem:

\begin{theorem} \label{thm:red}
There exists a polynomial time reduction that takes as input a regular instance $\cL = (V, \Sigma, \{\Pi_S\}_{S \in \cQ})$ of $k$-CSP and a real number $\nu > 0$, and produces a distribution $\D$ over $\B_d \times \{\pm 1\}$ and positive real numbers $\gamma, \kappa, \varepsilon, \alpha$ such that
\begin{itemize}
\item (Completeness) If $\cL$ is satisfiable, then $\opt_{\gamma}^{\D} \leq \kappa$.
\item (Soundness) If $\val(\cL) < \nu$, then $\opt_{0-1}^{\D} > \alpha \cdot \kappa + \varepsilon$.
\item (Margin Parameter) $\gamma = \Omega\left(\frac{1}{\Delta |\Sigma|^{3k}\sqrt{|\cQ|}}\right)$, where $\Delta$ denotes the number of constraints each variable appears in.
\item (Approximation Ratio) $\alpha = \Omega\left(\frac{(1/\nu)^{1/k}}{k}\right)$.
\item (Error Parameter) $\varepsilon = \Omega\left(\frac{1}{\Delta |\Sigma|^k}\right) \cdot \alpha$.
\item (Dimension) $d = n + 1$.
\end{itemize} 
\end{theorem}

\begin{proof}
Before we define $\cD$, let us specify the parameters:
\begin{itemize}
\item First, we let $d$ be $1 + n$. We name the first coordinate as $*$ and each of the remaining coordinates are named $(S, f)$ for a constraint $S \in \cQ$ and $f \in \Pi_S$.
\item Let $Z := 2\left(|V| \cdot |\Sigma| + 2k|\cQ| + 2k\sum_{e \in E} |\Pi_e|\right)$ be our ``normalizing factor'', which will be used below to normalized the probability.
\item Let $\delta := \frac{0.1}{\Delta |\Sigma|^{2k}}$ be the ``shift parameter''. Note that this is not the margin $\gamma$ (which will be defined below).
\item Let $s := 10\Delta|\Sigma|^k$ be the scaling factor, which we use to make sure that all our samples lie within the unit ball.
\item Let the gap parameter $\alpha$ be $\frac{(1/\nu)^{1/k}}{40k}$.
\item Finally, let $\kappa = \frac{|V|}{Z}$ and $\varepsilon = \kappa \cdot \alpha$.
\end{itemize}

Note that $\alpha$ as defined above can be less than one. However, this is not a problem: in the subsequent proofs of Theorems~\ref{thm:run-time} and~\ref{thm:inapx}, we will always choose the settings of parameters so that $\alpha > 1$. 

We are now ready to define the distribution $\cD$ on $\B_d \times \{\pm 1\}$, as follows:
\begin{enumerate}
\item Add a labeled sample $(-\be^*, -1)$ with probability $1/2$ to $\cD$. This corresponds to the constraint $\bw_* > 0$; we refer to this as the \emph{positivity constraint for $*$}.
\item Next, for each coordinate $(S, f)$, add a labeled sample $\left(\frac{1}{s}\left(\be^{(S, f)} + \delta \cdot \be^{*}\right), 1\right)$ with probability $2k/Z$ to $\cD$. This corresponds to $\bw_{(S, f)} + \delta \cdot \bw_{*} \geq 0$ scaled down by a factor of $1/s$ so that the vector is in the unit ball; we refer to this as the \emph{non-negativity constraint for $(S, f)$.} \label{bullet:pos}
\item For every $S \in \cQ$, add a labeled sample $\left(\frac{1}{s}\left(\sum_{f \in \Pi_S} \be^{(S, f)} - (1 - \delta) \be^*\right), 1\right)$ with probability $2k/Z$ to $\cD$. This corresponds to the constraint $\sum_{f \in \Pi_S} \bw_{(S, f)} \geq (1 - \delta) \bw_*$, scaled down by a factor of $1/s$. We refer to this constraint as the \emph{satisfiability constraint for $S$}. \label{bullet:sat}
\item For every variable $v \in V$ and $\sigma \in \Sigma$, add a labeled sample \\ $\left(\frac{1}{s}\left(\sum_{S \in \cQ: v \in S} \sum_{f \in \Pi_S: f(v) = \sigma} \be^{(S, f)} - \delta \be^{*}\right), -1\right)$ with probability $1/Z$ to $\cD$. This corresponds to the constraint $\sum_{S \in \cQ: v \in S} \sum_{f \in \Pi_S: f(v) = \sigma} \bw_{(S, f)} < \delta \cdot \bw_{*}$, scaled down by a factor of $1/s$. We refer to this as the \emph{selection constraint for $(v, \sigma)$}.
\end{enumerate}

\paragraph{Completeness.} Suppose that there exists an assignment $\phi: V \to \Sigma$ that satisfies all the constraints of $\cL$. Consider the halfspace with normal vector $\bw$ defined by $\bw_{*} = \zeta$ and 
\begin{align*}
\bw_{(S, f)} =
\begin{cases}
\zeta &\text{ if } f = \phi|_S, \\
0 &\text{ otherwise,}
\end{cases}
\end{align*}
where $\zeta := \frac{1}{\sqrt{1 + |\cQ|}}$ is the normalization factor. It is easy to see that the positivity constraints and the satisfiability constraints are satisfied with margin at least $\gamma = \zeta \cdot \delta/s = \Omega\left(\frac{1}{\Delta|\Sigma|^{3k}\sqrt{|\cQ|}}\right)$. Finally, observe that the sum $\sum_{S \in \cQ: v \in S} \sum_{f \in \Pi_S: f(v) = \sigma} \bw_{(S, f)}$ is zero if $f(v) \ne \sigma$; in this case, the selection constraint for $(v, \sigma)$ is also satisfied with margin at least $\gamma$. As a result, we only incur an error (with respect to margin $\gamma$) for the selection constraint for $(v, \phi(v))$ for all $v \in V$; hence, we have $\err_{\gamma}^{\cD}(\bw) \leq \frac{1}{Z} \cdot |V| = \kappa$ as desired.

\paragraph{Soundness.} Suppose contrapositively that there exists $\bw$ with $\err_{0-1}^{\D}(\bw) \leq \alpha \cdot \kappa + \varepsilon = 2\alpha\kappa$. We will ``decode'' back an assignment with value at least $\nu$ of the CSP from $\bw$.

To do so, first observe that from the positivity constraint for $*$, we must have $\bw_* > 0$, as otherwise we would already incur an error of $1/2 > 2\alpha\kappa$ with respect to $\cD$. Now, since scaling (by a positive factor) does not change the fraction of samples violated, we may assume w.l.o.g. that $\bw_* = 1$.

Next, we further claim that we may assume without loss of generality that $\bw$ does not violate any non-negativity constraints (\ref{bullet:pos}) or satisfiability constraints (\ref{bullet:sat}). The reason is that, if $\bw$ violates a non-negativity constraint for $(S = \{v_1, \dots, v_k\}, f)$, then we may simply change $\bw_{(S, f)}$ to zero. This reduces the error by $2k/Z$, while it may only additionally violate $k$ additional selection constraints for $(v_1, f(v_1)), \dots, (v_k, f(v_k))$ which weights $k/Z$ in total with respect to $\cD$. As a result, this change only reduces the error in total. Similarly, if the satisfiability constraint of $S$ is unsatisfied, we may change $\bw_{(S, f)}$ for some $f \in \Pi_S$ to a sufficiently large number so that this constraint is satisfied; once again, in total the error decreases. Hence, we may assume that the non-negativity constraints (\ref{bullet:pos}) and satisfiability constraints (\ref{bullet:sat}) all hold.

Now, for every vertex $v$, let $L_v \subseteq \Sigma$ denote the set of labels $\sigma \in \Sigma$ such that the selection constraint for $(v, \sigma)$ is violated. Since we assume that $\err_{0-1}^{\D}(\bw) \leq 2\alpha\kappa$, we must have $\sum_{v \in V} |L_v| \leq (2\alpha\kappa) / (1/Z) = 2\alpha|V|$.

Next, let $V_{\text{small}}$ denote the set of all variables $v \in V$ such that $|L_v| \leq 20\alpha k$. From the bound we just derived, we must have $|V_{\text{small}}| \geq \left(1 - \frac{1}{10k}\right)|V|$.

Another ingredient we need is the following claim:

\begin{claim} \label{claim:decode}
For every constraint $S = \{v_1, \dots, v_k\} \in \cQ$, there exist $\sigma_1 \in L_{v_1}, \dots, \sigma_k \in L_{v_k}$ that induces an accepting assignment for $\Pi_S$ (i.e., $f \in \Pi_S$ where $f$ is defined by $f(v_i) = \sigma_i$).
\end{claim}

\begin{proof}
Suppose for the sake of contradiction that no such $\sigma_1 \in L_{v_1}, \dots, \sigma_k \in L_{v_k}$ exists. In other words, for every $f \in \Pi_S$, there must exist $i \in [k]$ such that the selection constraint for $(v_i, f(v_i))$ is not violated. This means that
\begin{align*}
\delta = \delta \cdot \bw_{*} &> \sum_{S' \in \cQ: v \in S'} \sum_{f' \in \Pi_{S'}: f'(v) = \sigma} \bw_{(S', f')}  \\
&\geq \bw_{(S, f)} + \sum_{S' \in \cQ: v \in S'} \sum_{f' \in \Pi_{S'}: f'(v) = \sigma} -\delta \cdot \bw_* \\
&\geq \bw_{(S, f)} - \delta \cdot \Delta |\Sigma|^k \;,
\end{align*}
where the second inequality comes from our assumption, that the non-negativity constraints are satisfied.

Hence, by summing this up over all $f \in \Pi_S$, we get
\begin{align*}
\sum_{f \in \Pi_S} \bw_{(S, f)} \leq \delta \cdot (\Delta |\Sigma|^k + 1) \cdot |\Sigma|^k < (1 - \delta),
\end{align*}
which means that the satisfiability constraint for $S$ is violated, a contradiction.
\end{proof}

We can now define an assignment $\phi: V \to \Sigma$ for $\cL$ as follows. For every $v \in V$, let $\phi(v)$ be a random label in $L_v$. Notice here that, by Claim~\ref{claim:decode}, the probability that a constraint $S = \{v_1, \dots, v_k\}$ is satisfied is at least $\prod_{i \in [k]} |L_{v_i}|^{-1}$. Hence, the expected total number of satisfied constraints is at least
\begin{align*}
\sum_{S = \{v_1, \dots, v_k\} \in \cQ} \prod_{i \in [k]} |L_{v_i}|^{-1} &\geq \sum_{S = \{v_1, \dots, v_k\} \in \cQ: v_1, \dots, v_k \in V_{\text{small}}} \prod_{i \in [k]} |L_{v_i}|^{-1} \\
&\geq \sum_{S = \{v_1, \dots, v_k\} \in \cQ: v_1, \dots, v_k \in V_{\text{small}}} (20\alpha k)^{-k}.
\end{align*}
Recall that we have earlier bound $|V_{\text{small}}|$ to be at least $\left(1 - \frac{1}{10k}\right)|V|$. Hence, the fraction of constraints that involves some variable outside of $V_{\text{small}}$ is at most $\left(\frac{1}{10k}\right) \cdot (k) = 0.1$. Plugging this into the above inequality, we get that the expected total number of satisfied constraints is at least
\begin{align*}
0.9|\cQ| \cdot (20\alpha k)^{-k} > |\cQ| \cdot \nu,
\end{align*}
where the equality comes from our choice of $\alpha$. In other words, we have $\val(\cL) > \nu$ as desired.
\end{proof}

\subsubsection{Proofs of Theorems~\ref{thm:run-time},~\ref{thm:inapx} and Lemma~\ref{lem:strong-quantifier}} \label{sec:hardness-main-proofs}

We now prove Theorem~\ref{thm:run-time}, by simply applying Theorem~\ref{thm:red} with appropriate parameters on top of the Moshkovitz-Raz PCP.

\begin{proof}[Proof of Theorem~\ref{thm:run-time}]
Suppose contrapositively that, for some constant $\tilde{\alpha} \geq 1$ and $\zeta > 0$, we have an $O(2^{{(1/\gamma)}^{2-\zeta}}2^{d^{1-\zeta}})f(\frac{1}{\varepsilon})$ time $\tilde{\alpha}$-agnostic proper learner for $\gamma$-margin halfspaces.

Let $\nu > 0$ be a sufficiently small constant so that the parameter $\alpha$ (when $k = 2$) from Theorem~\ref{thm:red} is at least $\tilde{\alpha}$. (In particular, we pick $\nu = \frac{1}{C(\tilde{\alpha})^k}$ for some sufficiently large constant $C$.)

Given an instance $\cL$ of $\nu$-Gap-2-CSP, we run the reduction from Theorem~\ref{thm:red} to produce a distribution $\D$. We then run the learner on $\D$ with error parameter $\varepsilon$ as given by Theorem~\ref{thm:red} (and with $\delta = 1/3$). Note that the
learner runs in $O(2^{{(1/\gamma)}^{2-\zeta}}2^{d^{1-\zeta}})f(\frac{1}{\varepsilon}) = 2^{O(n^{1 - \zeta/2})}$ time, and produces a halfspace $h$. We compute $\err_{0-1}^{\D}(h)$; if it is no more than $\alpha \cdot \kappa + \varepsilon$, then we output YES. Otherwise, output NO.

The algorithm describe above solves $\nu$-Gap-2-CSP (correctly with probability 2/3) in $2^{O(n^{1 - \zeta/2})}$ time, which, by Theorem~\ref{thm:mr-pcp}, violates (randomized) ETH.
\end{proof}

Next, we prove Lemma~\ref{lem:strong-quantifier}. The main difference from the above proof is that, since the algorithm works only \emph{for some} margin $\gamma = \gamma(d)$. We will select the dimension $d$ to be as large as possible so that $\gamma(d)$ is still smaller than the margin given by Theorem~\ref{thm:red}. This dimension $d$ will be larger than the dimension given by Theorem~\ref{thm:red}; however, this is not an issue since we can simply ``pad'' the remaining dimensions by setting the additional coordinates to zeros. This is formalized below.

\begin{proof}[Proof of Lemma~\ref{lem:strong-quantifier}]
Let $\tilde{\alpha} \geq 1$ be any constant. Let $\nu > 0$ be a sufficiently small constant so that the parameter $\alpha$ (when $k = 2$) from Theorem~\ref{thm:red} is at least $\tilde{\alpha}$. (In particular, we pick $\nu = \frac{1}{C(\tilde{\alpha})^k}$ for some sufficiently large constant $C$.) Let $\varepsilon_0 = \varepsilon_0(\tilde{\alpha})$ be the parameter $\varepsilon$ given by Theorem~\ref{thm:red}.

Suppose contrapositively that, for some $\zeta > 0$, there is an $\tilde{\alpha}$-agnostic learner $\fA$ for $\gamma(\tilde{d})$-margin halfspaces that runs in time $O(2^{(1/\gamma)^{2 - \zeta}}) \poly(\tilde{d})$ for all dimensions $\tilde{d}$ and for some $0 < \varepsilon^* < \varepsilon_0(\alpha)$ and $\gamma(\tilde{d})$ that satisfies 
\begin{align} \label{eq:gamma-bound}
\frac{1}{\tilde{d}^{0.5 - \zeta}} \leq \gamma(\tilde{d}) \leq \frac{1}{(\log \tilde{d})^{0.5 + \zeta}}
\end{align} and 
\begin{align} \label{eq:gamma-consec}
\frac{\gamma(\tilde{d} + 1)}{\gamma(\tilde{d})} \geq \zeta.
\end{align}
We may assume without loss of generality that $\zeta < 0.1$.

We create an algorithm $\fB$ for $\nu$-Gap-2-CSP as follows:
\begin{itemize}
\item Given an instance $\cL$ of $\nu$-Gap-2-CSP of size $n$, we first run the reduction from Theorem~\ref{thm:red} with $\nu$ as selected above to produce a distribution $\cD$ on $\B_d \times \{\pm 1\}$ (where $d = n + 1$). Let the margin parameter $\gamma$ be as given in Theorem~\ref{thm:red}; observe that $\gamma = \Omega_{\nu}(1/\sqrt{n})$.
\item Let $\tilde{d}$ be the largest integer so that $\gamma(\tilde{d}) \geq \gamma$. Observe that, from the lower bound in~\eqref{eq:gamma-consec}, we have $\gamma(d) \geq \frac{1}{d^{0.5 - \zeta}}$. Hence, for a sufficiently large $d$, $\gamma(d)$ is larger than $\gamma$ (which is $O_{\nu}(1/\sqrt{d})$). In other words, we have $\tilde{d} \geq d$.
\item Create a distribution $\D'$ as follows: for each $(\bx, y) \in \supp(\D)$, we create a sample $(\bx', y)$ in $\D'$ with the same probability and where $\bx' \in \B_{\tilde{d}}$ is $\bx$ concatenated with $0$s in the last $\tilde{d} - d$ coordinates. 
\item Run the learner $\fA$ on $\D'$ with parameter $\gamma(\tilde{d})$ and $\varepsilon$. Suppose that it outputs a halfspace $h$. We compute $\err_{0-1}^{\D'}(h)$; if this is no more than $\alpha \cdot \kappa + \varepsilon_0(\alpha)$, then output YES. Otherwise, output NO.
\end{itemize}

It is simple to see that, in the completeness case, we must have $\opt_{\gamma(\tilde{d})}^{\D'} \leq \opt_{\gamma}^{\D'} = \opt_{\gamma}^{\D} \leq \kappa$; hence, $\fA$ would (with probability 2/3) output a halfspace $h$ with 0-1 error at most $\alpha \cdot \kappa + \varepsilon_0(\alpha)$, and we output YES. On the other hand, in the soundness case, we have $\opt_{0-1}^{\D'} = \opt_{0-1}^{\D'} > \alpha \cdot \kappa + \varepsilon_0(\tilde{\alpha})$, and we always output NO. Hence, the algorithm is correct with probability 2/3. 

Next, to analyze the running time of $\cB$, let us make a couple additional observations. First, from~\eqref{eq:gamma-consec}, we have
\begin{align}
\gamma(\tilde{d}) \leq \gamma / \zeta \leq O(1/\sqrt{n}).
\end{align} 
Furthermore, from the upper bound in~\eqref{eq:gamma-consec}, we have
\begin{align}
\tilde{d} \leq 2^{(1/\gamma(\tilde{d}))^{\frac{1}{0.5+\zeta}}} \leq 2^{O(n^{\frac{1}{1 + 2\zeta}})} \leq 2^{O(n^{1 - \zeta})},
\end{align}
where the last inequality follows from $\zeta < 0.1$.

As a result, the algorithm runs in time $O(2^{(1/\gamma(d))^{2 - \zeta}})\poly(\tilde{d}) \leq 2^{O(n^{1 - \zeta/2})}$, which from Theorem~\ref{thm:mr-pcp} would break the (randomized) ETH.
\end{proof}

Finally, we prove Theorem~\ref{thm:inapx}, which is again by simply applying Theorem~\ref{thm:red} to the Dinur-Harsha-Kindler PCP and the Sliding Scale Conjecture.

\begin{proof}[Proof of Theorem~\ref{thm:inapx}]
By plugging in our reduction from Theorem~\ref{thm:inapx} to Theorem~\ref{thm:dhk-pcp}, we get that it is NP-hard to, given a distribution $\D$, distinguish between $\opt_{\gamma}^{\D} \leq \kappa$ or $\opt_{0-1}^{\D} > \alpha \cdot \kappa + \Omega(\frac{1}{\poly(d)})$, 
where $\gamma = \frac{1}{d^{\polyloglog(d)}}$ and $\alpha = d^{1/\polyloglog(d)} = (1/\gamma)^{1/\polyloglog(1/\gamma)}$. 
In other words, if we have a polynomial time $\alpha$-agnostic learner for $\gamma$-margin halfspaces 
for this parameter regime, then NP = RP.

Similarly, by plugging in our reduction the Sliding Scale Conjecture, we get that it is NP-hard to, given a distribution $\D$, 
distinguish between $\opt_{\gamma}^{\D} \leq \kappa$ or $\opt_{0-1}^{\D} > \alpha \cdot \kappa + \Omega(\frac{1}{\poly(d)})$, 
where $\gamma = 1/d^{O(1)}$ and $\alpha = d^{\Omega(1)} = (1/\gamma)^{\Omega(1)}$. In other words, if we have 
a polynomial time $\alpha$-agnostic learner for $\gamma$-margin halfspaces for this parameter regime, then NP = RP.
\end{proof}




\new{

\section{Conclusions and Open Problems} \label{sec:conc}
This work gives nearly tight upper and lower bounds for the problem 
of $\alpha$-agnostic proper learning of halfspaces with a margin, for $\alpha = O(1)$. 
Our upper and lower bounds for $\alpha = \omega(1)$ are far from tight.
Closing this gap is an interesting open problem. Charactering the fine-grained complexity of 
the problem for improper learning algorithms remains a challenging open problem.

More broadly, an interesting direction for future work would be to generalize
our agnostic learning results to broader classes of geometric functions.
Finally, we believe that finding further connections between the problem of agnostic learning with a margin 
and adversarially robust learning is an intriguing direction to be explored.
}

\bibliographystyle{alpha}
\bibliography{allrefs}

\end{document}